\documentclass[11pt]{article}
\usepackage[preprint]{acl}
\usepackage{times}
\usepackage{latexsym}
\usepackage[T1]{fontenc}
\usepackage[utf8]{inputenc}
\usepackage{microtype}
\usepackage{inconsolata}
\usepackage{graphicx}
\usepackage{amsmath,amssymb,amsthm}
\usepackage{booktabs}
\usepackage{enumitem}
\usepackage{makecell}
\usepackage{multirow}
\usepackage{array}
\usepackage{xcolor}
\usepackage{tikz}
\usetikzlibrary{arrows.meta,positioning,calc}
\usepackage{algorithm}
\usepackage[noend]{algpseudocode}
\usepackage{flafter}
\graphicspath{{figures/}{../figures/}{./}}
\makeatletter\g@addto@macro\UrlBreaks{\do\-}\makeatother

\newtheorem{theorem}{Theorem}[section]

\newcommand{\Auth}{\mathrm{Auth}}
\newcommand{\Cl}{\mathrm{Cl}}
\newcommand{\MaxAuth}{\mathrm{MaxAuth}}
\newcommand{\EMIT}{\textsc{emit}}
\newcommand{\stU}{\mathsf{U}}
\newcommand{\stT}{\mathsf{T}}
\newcommand{\stF}{\mathsf{F}}
\newcommand{\stC}{\mathsf{C}}
\newcommand{\cP}{\mathcal{P}}

\newcommand{\cA}{\mathcal{A}}
\newcommand{\cR}{\mathcal{R}}

\newcommand{\denote}{\mathrm{denote}}
\newcommand{\Render}{\mathrm{Render}}
\newcommand{\Ev}[1]{E_{\mathrm{#1}}}
\newcommand{\by}[1]{\ \text{\footnotesize[#1]}}

\title{PRIME: Towards Verifiable Evidence-Grounded Reasoning in LLM Agents}

\author{Hamed Khosravi \qquad Xiaoming Huo \\
  H. Milton Stewart School of Industrial and Systems Engineering \\
  Georgia Institute of Technology \\
  \texttt{hkhosravi7@gatech.edu} \qquad \texttt{huo@gatech.edu}}

\begin{document}
\maketitle

\begin{abstract}
LLM agents are used to check records against an organization's requirements, for example whether a vulnerability report is ready for investigation. Requirements often allow several ways to be met, such as a severity rating with a curator-reviewed advisory or a published advisory link. Agents fail by accepting a report after a required check has failed, or by settling for a weaker conclusion without trying another way that could still succeed. We introduce PRIME (Precommitted Rules for Inference from Minimal Evidence). Before checking, it fixes evidence routes specifying how each conclusion may be justified. A separate program records each check result and refuses any conclusion the results do not justify. It also points the agent to checks that could justify a stronger conclusion, or runs them itself. PRIME reports only conclusions authorized by its rules and ledger. With correct requirements, exact checks at known locations, and sufficient budget at stopping, completion reaches the strongest supported conclusion. On 480 public-record tasks, an agent shown the requirements reported unjustified conclusions on 26.7\% of tasks, while PRIME reported none. Its directions and added checks raised the average score for justified conclusions from 30.1, with refusal alone, to the maximum of 59.2 out of 100.
\end{abstract}

\section{Introduction}
\label{sec:intro}

Large language model (LLM) agents increasingly assist security triage \citep{saju2026soc,liu2026hesp}, and similar agents can check records against an organization's requirements, for example whether a software vulnerability report has enough information for analysts to investigate. The outcome is rarely a simple yes or no. A review can reach conclusions of increasing strength, from ``the report has a description'' to ``the report is ready for investigation'', and requirements often accept more than one combination of checks. Suppose a company accepts a report for investigation if it has a description and either (a) a severity rating together with an advisory reviewed by a curator, or (b) a link to a published security advisory. An agent finds that the curator review is missing. It can then fail in two ways. It may accept the report anyway, an \emph{unsupported conclusion} that its checks do not justify. Or it may settle for a weaker conclusion without checking the link, a \emph{missed conclusion}, although option (b) could still succeed. Showing an agent its requirements and previous check results does not enforce what it may report.

Existing methods address one failure or the other. Self-verification rechecks answers \citep{dhuliawala2023cove}, citation methods link claims to sources \citep{gao2023alce}, and conformal methods calibrate factuality risk \citep{mohri2024conformal}. None of them requires a reported conclusion to rest on checks that were actually run. Rule-enforcement approaches block agent actions that break fixed rules \citep{wang2026agentspec,doshi2026safe}, and evidence-gated systems refuse a claim until evidence for it has been recorded \citep{researchloop2026,prooforstop2026}. They do not jointly use fixed alternative routes to authorize graded conclusions and guide checking. Evidence-gathering methods select checks \citep{zhang2026glean,liu2026hesp}, but do not combine both functions through fixed routes.

We introduce PRIME (Precommitted Rules for Inference from Minimal Evidence), which handles both failures with one set of requirements. Before review, PRIME fixes the combinations of checks its program treats as sufficient for each conclusion. We call each such combination an \emph{evidence route}. A component separate from the agent, the \emph{runtime}, then records every check result and refuses to report any conclusion that lacks a fully checked route. It also tells the agent which unchecked items could still complete a route to the strongest conclusion within the remaining tool calls, and if the agent stops early, the runtime can run those checks itself. In the example, PRIME refuses to accept the report after the failed curator check and directs the agent to check the advisory link.

Our contributions follow the structure of the paper.
\begin{itemize}[leftmargin=*,itemsep=0pt,topsep=2pt,parsep=0pt]
\item \textbf{Formulation.} We state the review as choosing checks under a limit on tool calls, with the strict rule that every reported conclusion must be justified by checks actually run (Section~\ref{sec:problem}).
\item \textbf{Method.} PRIME uses the same evidence routes to decide what may be reported and which checks to run next (Section~\ref{sec:method}).
\item \textbf{Guarantees.} No agent behavior can make PRIME show the user an unjustified conclusion, and its remaining errors come only from wrong requirements or wrong check results. With correct rules, exact checks at known locations, and sufficient budget at stopping, completion reaches the strongest supported conclusion (Section~\ref{sec:guarantees}).
\item \textbf{Positioning.} We compare PRIME with verification, enforcement, and evidence-gathering systems and state its costs (Section~\ref{sec:positioning}).
\item \textbf{Benchmark and evaluation.} SUFFICE has 938 tasks built from public security, clinical-trial, drug-recall, and financial records, scored by code rather than by an LLM (Section~\ref{sec:bench}). On 480 held-out tasks and four models, PRIME reported no unsupported conclusions. For Qwen2.5-32B with refusal alone, the average score for supported conclusions was 30.1 out of 100. PRIME's directions raised it to 45.5, and running the remaining checks raised it to 59.2, the maximum possible on these tasks. Programs generated by an LLM from a clinical-trial regulation also ran unchanged on real registry records (Section~\ref{sec:experiments}).
\end{itemize}

\section{Problem Formulation}
\label{sec:problem}

\paragraph{Review task.} A review case $x=(S_x,V_x)$ has a \emph{specification} $S_x$ naming conclusions and available checks, and an initially unchecked \emph{evidence view} $V_x$ containing the report's values.
\begin{itemize}[leftmargin=*,itemsep=0pt,topsep=2pt,parsep=0pt]
\item \emph{Conclusions} $C=\{c_1\prec\dots\prec c_K\}$, ordered by strength, with utilities $0=U(\bot)<U(c_1)<\dots<U(c_K)$, where $\bot$ is abstention.
\item \emph{Evidence atoms} $E=\{e_1,\dots,e_m\}$, yes-or-no conditions such as ``a severity rating exists'', with hidden values $\omega_x\in\{0,1\}^m$.
\item \emph{Checks}: field and text reads and \emph{attestations}, in which a model or person records a judgment as evidence, settle one or more atoms, and each atom's \emph{resolver} is the registered procedure that checks it. Uncached checks and emissions cost one call from a budget $B$; cached repeats and $\textsc{stop}$ are free (Appendix~\ref{app:formal}).
\end{itemize}

\paragraph{Requirements.} A \emph{review program} $P$ gives each conclusion $c$ a set of evidence routes $P(c)$. A route $p$ is a conjunction of literals $e$ or $\neg e$, and completing any one route justifies $c$. In the example, acceptance has the routes $p_1=\{\text{description},\text{severity},\text{curator review}\}$ and $p_2=\{\text{description},\text{advisory reference},\text{URL}\}$. The \emph{intended standard} $f^\star_x(c,\omega)\in\{0,1\}$ says whether a complete assignment $\omega$ justifies $c$. The program is \emph{correct} when, for every $\omega$ and $c$, $f^\star_x(c,\omega)=1$ exactly when some route of $c$ or of a stronger conclusion holds under $\omega$.

\paragraph{Evidence and authorization.} The \emph{ledger} $z\in\{\stU,\stT,\stF\}^m$ records each atom as unknown, established true, or established false. A route $p$ holds at $z$, written $z\models p$, if every literal of $p$ is established. A conclusion is \emph{authorized} if a route to it or to a stronger conclusion holds:
\begin{equation}\label{eq:authorization}
\Auth_P(z)=\{c\mid\exists d\succeq c,\ p\in P(d):\ z\models p\}.
\end{equation}
Authorization compares recorded evidence with the program, not with the world: it does not test whether the program or the observations are correct.

\paragraph{Decision problem.} The agent issues one request per step: run a check (which updates the ledger), report a conclusion with $\EMIT(c)$, or $\textsc{stop}$. Let $P_0$ be the program fixed before the first check, $\pi$ the agent's policy, $\hat c$ the reported conclusion, and $z_\tau$ the final ledger. The goal is
\begin{equation}\label{eq:problem}
\begin{aligned}
\max_{\pi}\quad&\mathbb E_\pi\big[U(\hat c)\,f^\star_x(\hat c,\omega_x)\big]\\
\mathrm{s.t.}\quad&\hat c\in\Auth_{P_0}(z_\tau)\cup\{\bot\}\ \text{for every episode},\\
&\textstyle\text{total cost}\le B .
\end{aligned}
\end{equation}
The constraint makes a report \emph{verifiable}: anyone holding $P_0$ and the ledger can check it without trusting the agent. It must hold in every episode, not on average. The objective rewards the strongest conclusion that is also true. An unsupported conclusion violates the constraint, and a missed conclusion loses utility. Appendix~\ref{app:formal} adds logical constraints $\Phi_x$ between atoms, conflicting observations, and partially ordered conclusions, and Appendix~\ref{app:setup} gives the exact charging rule.

\section{PRIME}
\label{sec:method}

PRIME solves Problem~\eqref{eq:problem} in two parts: it enforces the constraint exactly and pursues the objective heuristically. Both use the committed routes (Figure~\ref{fig:overview}).

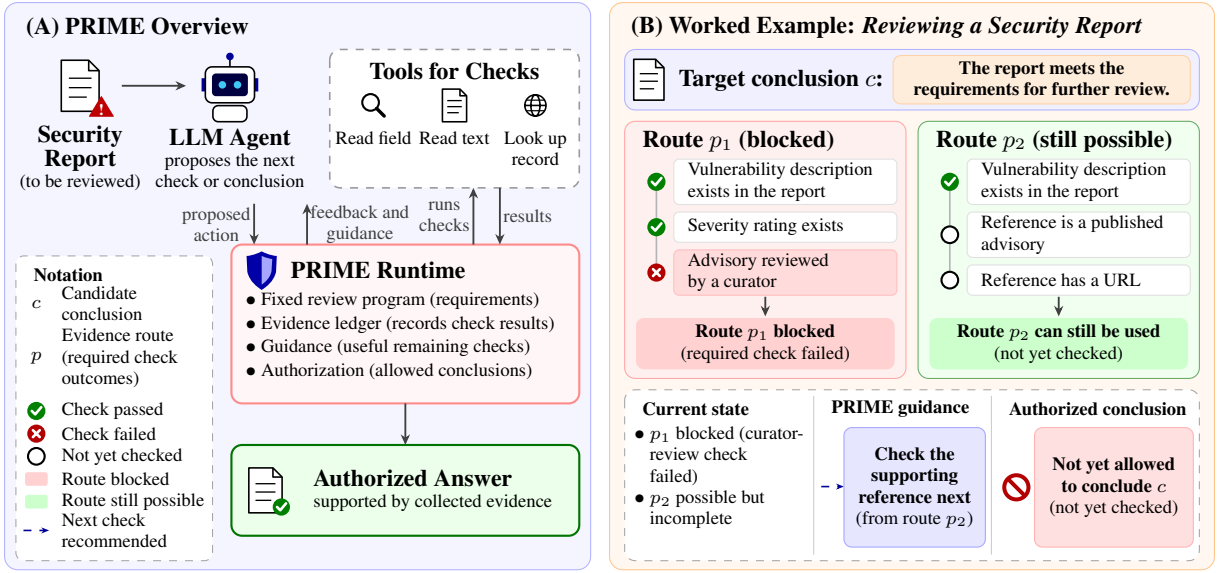
\begin{figure*}[t]
\centering
\tikzset{
  fdoc/.pic={
    \draw[fill=white, line width=0.6pt, line join=round] (-0.2,-0.27) -- (-0.2,0.27) -- (0.09,0.27) -- (0.2,0.16) -- (0.2,-0.27) -- cycle;
    \draw[line width=0.6pt, line join=round] (0.09,0.27) -- (0.09,0.16) -- (0.2,0.16);
    \draw[line width=0.5pt] (-0.12,0.08) -- (0.12,0.08) (-0.12,-0.02) -- (0.12,-0.02) (-0.12,-0.12) -- (0.07,-0.12);
  },
  fwarn/.pic={
    \fill[red!75!black, rounded corners=0.4pt] (0,0.15) -- (0.17,-0.13) -- (-0.17,-0.13) -- cycle;
    \node[text=white, font=\bfseries\fontsize{6}{6}\selectfont, inner sep=0] at (0,-0.035) {!};
  },
  frobot/.pic={
    \draw[line width=0.6pt] (0,0.2) -- (0,0.29);
    \fill[blue!60!black] (0,0.31) circle (0.04);
    \fill[blue!60!black] (-0.3,-0.06) rectangle (-0.26,0.08) (0.26,-0.06) rectangle (0.3,0.08);
    \draw[fill=white, line width=0.6pt, rounded corners=2.5pt] (-0.26,-0.18) rectangle (0.26,0.2);
    \fill[blue!25!black, rounded corners=1.5pt] (-0.19,-0.11) rectangle (0.19,0.13);
    \fill[cyan!60] (-0.08,0.01) circle (0.04) (0.08,0.01) circle (0.04);
    \draw[fill=white, line width=0.6pt, rounded corners=2pt] (-0.18,-0.38) rectangle (0.18,-0.21);
  },
  fshield/.pic={
    \fill[blue!65!black] (0,0.24) .. controls (0.1,0.19) and (0.17,0.18) .. (0.2,0.18) -- (0.19,-0.03)
      .. controls (0.16,-0.16) and (0.06,-0.23) .. (0,-0.27) .. controls (-0.06,-0.23) and (-0.16,-0.16) .. (-0.19,-0.03)
      -- (-0.2,0.18) .. controls (-0.17,0.18) and (-0.1,0.19) .. cycle;
    \fill[blue!35] (0,0.18) .. controls (0.07,0.15) and (0.12,0.14) .. (0.14,0.14) -- (0.13,-0.02)
      .. controls (0.11,-0.12) and (0.04,-0.17) .. (0,-0.2) -- cycle;
  },
  fsearch/.pic={
    \draw[line width=0.8pt] (-0.04,0.04) circle (0.11);
    \draw[line width=1.3pt, line cap=round] (0.04,-0.04) -- (0.14,-0.14);
  },
  fglobe/.pic={
    \draw[line width=0.6pt] (0,0) circle (0.14) (0,0) ellipse (0.06 and 0.14) (-0.14,0) -- (0.14,0);
    \draw[line width=0.45pt] (-0.12,0.07) -- (0.12,0.07) (-0.12,-0.07) -- (0.12,-0.07);
  },
  fok/.pic={
    \fill[green!50!black] (0,0) circle (0.12);
    \draw[white, line width=1pt, line cap=round, line join=round] (-0.06,0.0) -- (-0.015,-0.045) -- (0.065,0.05);
  },
  ffail/.pic={
    \fill[red!70!black] (0,0) circle (0.12);
    \draw[white, line width=1pt, line cap=round] (-0.045,-0.045) -- (0.045,0.045) (-0.045,0.045) -- (0.045,-0.045);
  },
  fopen/.pic={\draw[fill=white, line width=0.8pt] (0,0) circle (0.11);},
  fnogo/.pic={\draw[red!70!black, line width=1.5pt] (0,0) circle (0.16) (-0.11,0.11) -- (0.11,-0.11);},
}
\begin{tikzpicture}[
  x=1cm, y=1cm, font=\scriptsize,
  arr/.style={-{Stealth[length=4pt,width=3.5pt]}, line width=0.7pt, black!75},
  lab/.style={font=\scriptsize, align=center, inner sep=1pt},
  card/.style={draw=black!12, fill=white, rounded corners=2pt},
]
\draw[rounded corners=5pt, draw=blue!30, fill=blue!4] (0,0) rectangle (7.75,-7.5);
\node[anchor=west, font=\small\bfseries] at (0.15,-0.33) {(A) PRIME Overview};
\pic[scale=1.25] at (1.0,-1.1) {fdoc};
\pic at (1.28,-1.38) {fwarn};
\node[lab, anchor=north] at (1.0,-1.58) {{\footnotesize\bfseries Security}\\ {\footnotesize\bfseries Report}\\ (to be reviewed)};
\draw[arr] (1.55,-1.1) -- (2.4,-1.1);
\pic[scale=1.25] at (2.98,-1.1) {frobot};
\node[lab, anchor=north] at (2.98,-1.62) {{\footnotesize\bfseries LLM Agent}\\ proposes the next\\ check or conclusion};
\draw[rounded corners=4pt, dashed, draw=black!40, fill=white, line width=0.6pt] (4.3,-0.62) rectangle (7.6,-2.45);
\node[font=\footnotesize\bfseries] at (5.95,-0.9) {Tools for Checks};
\pic at (4.88,-1.36) {fsearch};
\pic[scale=0.8] at (5.95,-1.36) {fdoc};
\pic at (7.02,-1.36) {fglobe};
\node[lab, anchor=north] at (4.88,-1.66) {Read field};
\node[lab, anchor=north] at (5.95,-1.66) {Read text};
\node[lab, anchor=north] at (7.02,-1.66) {Look up\\ record};
\draw[rounded corners=4pt, draw=red!45, fill=red!3, line width=0.8pt] (3.0,-3.2) rectangle (7.6,-5.3);
\pic at (3.4,-3.53) {fshield};
\node[anchor=west, font=\footnotesize\bfseries] at (3.65,-3.53) {PRIME Runtime};
\node[anchor=north west, align=left, inner sep=0] at (3.18,-3.82) {%
  $\bullet$ Fixed review program (requirements)\\[1pt]
  $\bullet$ Evidence ledger (records check results)\\[1pt]
  $\bullet$ Guidance (useful remaining checks)\\[1pt]
  $\bullet$ Authorization (allowed conclusions)};
\draw[arr] (3.3,-2.66) -- (3.3,-3.2) node[midway, left, lab] {proposed\\ action};
\draw[arr] (4.0,-3.2) -- (4.0,-2.66) node[midway, right, lab] {feedback and\\ guidance};
\draw[arr] (6.2,-3.2) -- (6.2,-2.45) node[midway, left, lab] {runs\\ checks};
\draw[arr] (6.55,-2.45) -- (6.55,-3.2) node[midway, right, lab] {results};
\draw[arr] (5.3,-5.3) -- (5.3,-5.85);
\draw[rounded corners=4pt, draw=green!45!black, fill=green!8, line width=0.8pt] (3.0,-5.85) rectangle (7.6,-7.05);
\pic[scale=1.1] at (3.45,-6.45) {fdoc};
\pic[scale=0.95] at (3.66,-6.67) {fok};
\node[anchor=west, align=left] at (3.95,-6.45) {{\footnotesize\bfseries Authorized Answer}\\ supported by collected evidence};
\draw[rounded corners=3pt, dashed, draw=black!35, fill=white] (0.15,-3.35) rectangle (2.75,-7.38);
\node[anchor=west, font=\scriptsize\bfseries] at (0.25,-3.6) {Notation};
\node at (0.42,-3.98) {$c$};      \node[anchor=west, align=left] at (0.62,-3.98) {Candidate\\ conclusion};
\node at (0.42,-4.70) {$p$};      \node[anchor=west, align=left] at (0.62,-4.70) {Evidence route\\ (required check\\ outcomes)};
\pic at (0.42,-5.40) {fok};       \node[anchor=west] at (0.62,-5.40) {Check passed};
\pic at (0.42,-5.70) {ffail};     \node[anchor=west] at (0.62,-5.70) {Check failed};
\pic at (0.42,-6.00) {fopen};     \node[anchor=west] at (0.62,-6.00) {Not yet checked};
\fill[red!18, rounded corners=1pt] (0.27,-6.21) rectangle (0.57,-6.39); \node[anchor=west] at (0.62,-6.30) {Route blocked};
\fill[green!20, rounded corners=1pt] (0.27,-6.51) rectangle (0.57,-6.69); \node[anchor=west] at (0.62,-6.60) {Route still possible};
\draw[blue!60!black, dashed, -{Stealth[length=3pt]}, line width=0.6pt] (0.25,-6.98) -- (0.59,-6.98);
\node[anchor=west, align=left] at (0.62,-6.98) {Next check\\ recommended};
\draw[rounded corners=5pt, draw=orange!45, fill=orange!6] (8.0,0) rectangle (16.0,-7.5);
\node[anchor=west, font=\small\bfseries] at (8.15,-0.33) {(B) Worked Example: \textit{Reviewing a Security Report}};
\draw[rounded corners=4pt, draw=blue!30, fill=blue!6] (8.2,-0.62) rectangle (15.8,-1.42);
\pic at (8.52,-1.02) {fdoc};
\node[anchor=west, font=\footnotesize\bfseries] at (8.78,-1.02) {Target conclusion $c$:};
\draw[rounded corners=3pt, draw=orange!35, fill=orange!14] (11.75,-0.7) rectangle (15.68,-1.34);
\node[font=\scriptsize\bfseries, align=center] at (13.715,-1.02) {The report meets the\\ requirements for further review.};
\draw[rounded corners=4pt, draw=red!35, fill=red!6] (8.2,-1.57) rectangle (11.92,-4.97);
\node[anchor=west, font=\footnotesize\bfseries] at (8.3,-1.83) {Route $p_1$ (blocked)};
\draw[black!30, line width=0.6pt] (8.62,-2.37) -- (8.62,-3.57);
\draw[card] (8.85,-2.07) rectangle (11.8,-2.67); \pic at (8.62,-2.37) {fok};
\node[anchor=west, align=left] at (8.9,-2.37) {Vulnerability description\\ exists in the report};
\draw[card] (8.85,-2.77) rectangle (11.8,-3.17); \pic at (8.62,-2.97) {fok};
\node[anchor=west] at (8.9,-2.97) {Severity rating exists};
\draw[rounded corners=2pt, draw=red!25, fill=red!14] (8.85,-3.27) rectangle (11.8,-3.87); \pic at (8.62,-3.57) {ffail};
\node[anchor=west, align=left] at (8.9,-3.57) {Advisory reviewed\\ by a curator};
\draw[arr] (10.06,-3.89) -- (10.06,-4.15);
\draw[rounded corners=3pt, fill=red!18, draw=none] (8.35,-4.17) rectangle (11.77,-4.85);
\node[align=center] at (10.06,-4.51) {{\bfseries Route $p_1$ blocked}\\ (required check failed)};
\draw[rounded corners=4pt, draw=green!45!black!60, fill=green!6] (12.08,-1.57) rectangle (15.8,-4.97);
\node[anchor=west, font=\footnotesize\bfseries] at (12.18,-1.83) {Route $p_2$ (still possible)};
\draw[black!30, line width=0.6pt] (12.5,-2.37) -- (12.5,-3.67);
\draw[card] (12.73,-2.07) rectangle (15.68,-2.67); \pic at (12.5,-2.37) {fok};
\node[anchor=west, align=left] at (12.78,-2.37) {Vulnerability description\\ exists in the report};
\draw[card] (12.73,-2.77) rectangle (15.68,-3.37); \pic at (12.5,-3.07) {fopen};
\node[anchor=west, align=left] at (12.78,-3.07) {Reference is a published\\ advisory};
\draw[card] (12.73,-3.47) rectangle (15.68,-3.87); \pic at (12.5,-3.67) {fopen};
\node[anchor=west] at (12.78,-3.67) {Reference has a URL};
\draw[arr] (13.94,-3.89) -- (13.94,-4.15);
\draw[rounded corners=3pt, fill=green!20, draw=none] (12.23,-4.17) rectangle (15.65,-4.85);
\node[align=center] at (13.94,-4.51) {{\bfseries Route $p_2$ can still be used}\\ (not yet checked)};
\draw[rounded corners=4pt, dashed, draw=black!35, fill=white] (8.2,-5.12) rectangle (15.8,-7.38);
\draw[black!20] (10.7,-5.2) -- (10.7,-7.3) (13.05,-5.2) -- (13.05,-7.3);
\node[anchor=west, font=\scriptsize\bfseries] at (8.3,-5.37) {Current state};
\node[anchor=north west, align=left, inner sep=0] at (8.32,-5.6) {%
  $\bullet$ $p_1$ blocked (curator-\\ \phantom{$\bullet$} review check\\ \phantom{$\bullet$} failed)\\
  $\bullet$ $p_2$ possible but\\ \phantom{$\bullet$} incomplete};
\node[anchor=west, font=\scriptsize\bfseries] at (10.8,-5.37) {PRIME guidance};
\draw[blue!60!black, dashed, -{Stealth[length=3pt]}, line width=0.6pt] (10.8,-6.4) -- (11.1,-6.4);
\draw[rounded corners=3pt, draw=blue!35, fill=blue!10] (11.1,-5.62) rectangle (12.97,-7.2);
\node[align=center] at (12.035,-6.41) {\textbf{Check the}\\ \textbf{supporting}\\ \textbf{reference next}\\ (from route $p_2$)};
\node[anchor=west, font=\scriptsize\bfseries] at (13.15,-5.37) {Authorized conclusion};
\pic at (13.38,-6.41) {fnogo};
\draw[rounded corners=3pt, draw=red!30, fill=red!12] (13.62,-5.62) rectangle (15.72,-7.2);
\node[align=center] at (14.67,-6.41) {\textbf{Not yet allowed}\\ \textbf{to conclude} $c$\\ (not yet checked)};
\end{tikzpicture}
\caption{PRIME in a hypothetical security-report review. (A) The independent runtime records check outcomes, uses precommitted evidence routes to guide the agent, and authorizes conclusions. (B) Acceptance requires a description plus either a severity rating and curator-reviewed advisory ($p_1$) or a published advisory reference with a URL ($p_2$). A failed curator-review check blocks $p_1$, while the reference checks of $p_2$ remain open. PRIME recommends checking the reference but cannot yet authorize acceptance.}
\label{fig:overview}
\end{figure*}

\paragraph{1. Commit.} A commitment procedure $G$ produces $P_0=G(S_x,r)$ from the specification and randomness $r$ alone, never from field values. The program may be supplied by a person or generated by a model. The runtime validates its structure, records its hash, and only then enables checks, so a failed check cannot lead to rewriting the requirements (Appendix~\ref{app:formal}).

\paragraph{2. Gate.} The runtime refuses $\EMIT(c)$ unless $c\in\Auth_{P_0}(z)$ at the current ledger. If checking ends without an accepted request, it reports the strongest authorized conclusion, or abstains if none is authorized. A renderer with fixed templates produces the answer, listing the conclusion with its supporting checks rather than free model prose.

\paragraph{3. Guide.} Blocking unsupported reports can make the agent stop too early, so PRIME also uses the routes to direct further checking. The \emph{residual} of route $p$ at $z$ is the set of its literals not yet established,
\begin{equation}
\label{eq:residual}
R(p,z)=\{\ell\in p\mid z\not\models\ell\}.
\end{equation}
A route is \emph{live} if none of its literals is established false or certified unavailable. Let $\hat\kappa(p,z)$ be the number of distinct checks needed to settle $R(p,z)$. With remaining budget $b$, PRIME targets the strongest conclusion that has a live route within budget, and that conclusion's cheapest live route:
\begin{equation}
\label{eq:target}
\begin{aligned}
\cR(z,b)&=\{c\mid\exists p\in P_0(c)\text{ live},\ \hat\kappa(p,z)\le b\},\\
c^\star&\in\arg\max_{c\in\cR(z,b)}U(c),\\
p^\star&\in\arg\min_{p\in P_0(c^\star)\text{ live}}\hat\kappa(p,z).
\end{aligned}
\end{equation}
If $\cR(z,b)$ is empty, $c^\star=\bot$. The agent sees $c^\star$, $R(p^\star,z)$, and $\Auth_{P_0}(z)$, and still chooses its own next check. While $c^\star$ is stronger than every authorized conclusion, the runtime holds requests to report a weaker one and answers up to three $\textsc{stop}$ requests with reminders.

\paragraph{Completion.} With completion enabled, when the agent finally stops, the runtime runs the open checks of $p^\star$ itself and recomputes the target after each result, until the target is authorized, no stronger conclusion is reachable, or the budget is spent. It uses only registered resolvers. It cannot invent evidence or make judgments that require a model or a person. Algorithm~\ref{alg:prime} in Appendix~\ref{app:proofs} gives the full procedure.

\section{Guarantees and Why They Matter}
\label{sec:guarantees}

We assume a trusted runtime that faithfully records resolver outputs and identified attestations, and a renderer that is \emph{sound}: its text expresses only conclusions marked as authorized. Appendix~\ref{app:formal} states these premises (T1--T4), and Appendix~\ref{app:proofs} gives all proofs.

\begin{theorem}[Authorization holds for every agent]
\label{thm:safety}
(i) For fixed randomness, two cases with the same specification receive the same committed program, whatever their evidence. (ii) For every policy, sequence of actions and resolver results, budget, and stopping point, every conclusion PRIME outputs at final ledger $z$ lies in $\Auth_{P_0}(z)$. (iii) If the renderer is sound, every conclusion the user sees also lies in $\Auth_{P_0}(z)$.
\end{theorem}

\emph{Why it matters.} The constraint of~\eqref{eq:problem} holds whatever the model, prompt, or agent behavior, so it does not depend on the agent following instructions. An auditor can check any answer against the committed program and the ledger without rerunning the agent. Part (i) rules out tailoring the requirements to the evidence.

\medskip
Authorization is not truth. Call a ledger \emph{truthful} when every atom it settles agrees with $\omega_x$. Let $\Ev{prog}$ be the event that the program authorizes, at some truthful ledger, a conclusion with $f^\star_x=0$, and $\Ev{res}$ the event that the runtime admits an incorrect atom value, from a resolver or an attester. Let $\Ev{mask}$ and $\Ev{render}$ be the failures of parts (ii) and (iii), and $\Ev{user}$ the event that the user sees a conclusion with $f^\star_x=0$.

\begin{theorem}[Where end-to-end error comes from]
\label{thm:decomposition}
Under T3 and T4, for any distribution of episodes,
\begin{equation}
\begin{aligned}
\Pr(\Ev{user})\le\ &\Pr(\Ev{prog})+\Pr(\Ev{res})\\
&+\Pr(\Ev{mask})+\Pr(\Ev{render}),
\end{aligned}
\label{eq:decomp}
\end{equation}
and by Theorem~\ref{thm:safety} the last two terms are zero.
\end{theorem}

\emph{Why it matters.} Any remaining error comes from two components that can be tested separately: the rules, which can be checked against the intended standard, and the observations, which can be verified. Effort can therefore go where error actually arises. When every atom is settled by an exact deterministic resolver, $\Pr(\Ev{res})=0$. Wrong admitted judgments still enter $\Ev{res}$.

\medskip
The first two results concern safety. The third concerns utility: it says when the gate costs nothing. Let $c^\dagger$ be the strongest conclusion with $f^\star_x(c^\dagger,\omega_x)=1$, or $\bot$ if there is none. Let $E^\dagger$ be the atoms on the routes of $c^\dagger$ and of every stronger conclusion, or of all conclusions if $c^\dagger=\bot$. Let $N(z)$ be the number of distinct checks needed to settle the atoms of $E^\dagger$ that are unknown at $z$, and $z_0$ the initial ledger.

\begin{theorem}[Completion reaches the strongest supported conclusion]
\label{thm:ceiling}
Suppose the conclusions form a chain, $P_0$ is correct, $\Phi_x$ is empty, no atom of $E^\dagger$ has withheld evidence, every check returns the true values of the atoms it settles and reads a known field, and completion is enabled. If the agent's checking ends without an accepted report at a ledger $z$ with remaining budget $b\ge N(z)$, PRIME reports $c^\dagger$. In particular, if $B\ge N(z_0)$, PRIME reports $c^\dagger$ for an agent that stops at once, and for any agent whose charged requests all check unknown atoms of $E^\dagger$.
\end{theorem}

\emph{Proof idea.} Because checks are exact, the ledger never contradicts $\omega_x$. So nothing above $c^\dagger$ is ever authorized, and some route $p^\dagger$ of $c^\dagger$, true under $\omega_x$, is never blocked. While $b\ge N(z)$, $p^\dagger$ fits the budget, so the target $c^\star$ is $c^\dagger$ or stronger. Its route then lies in $E^\dagger$, and each completion check lowers $b$ and $N(z)$ by one, which preserves $b\ge N(z)$. Completion therefore ends either with $c^\dagger$ authorized or with every atom of $E^\dagger$ settled, which also authorizes $c^\dagger$ (Appendix~\ref{app:proofs}).

\emph{The bound is tight, and $b_{\min}$ is not enough.} Let $b_{\min}$ be the cost of the cheapest route to the strongest conclusion. Take $c_1\prec c_2\prec c_3$ with routes $\{e_1\}$ or $\{e_3\}$ for $c_1$, $\{e_1,e_2\}$ or $\{e_3,e_4\}$ for $c_2$, and $\{e_1,e_5,e_2\}$ for $c_3$, checked in that order, where only $e_2$ is false. The target rule first pursues $c_3$, spends three checks learning that $e_2$ fails, and only then turns to $\{e_3,e_4\}$. With 3 calls, which equals $b_{\min}$ and exceeds the 2-call cost of $c_2$'s true route, PRIME ends at $c_1$. With 4 calls, one fewer than $N(z_0)$, it still ends at $c_1$, and with $N(z_0)=5$ calls it reaches $c_2$.

\emph{Why it matters.} Under these premises, whenever the agent's checking ends without an accepted report and with $b\ge N(z)$, PRIME reaches $c^\dagger$ and never reports above it (Theorem~\ref{thm:decomposition} with $\Pr(\Ev{prog})=\Pr(\Ev{res})=0$), so it earns the largest utility that Problem~\eqref{eq:problem} allows. The bound gives a budget rule that can be computed before checking, since $N(z_0)$ is at most the number of distinct checks on all routes. The guarantee requires sufficient remaining budget and known evidence locations; it does not cover early stopping without completion or checks requiring field discovery.

\section{Positioning, Advantages, and Costs}
\label{sec:positioning}

\begin{table}[t]
\centering
\fontsize{8}{9.5}\selectfont
\setlength{\tabcolsep}{2pt}
\begin{tabular}{@{}>{\raggedright\arraybackslash}p{0.33\columnwidth}>{\raggedright\arraybackslash}p{0.3\columnwidth}>{\raggedright\arraybackslash}p{0.33\columnwidth}@{}}
\toprule
Approach & Controls & Next check \\
\midrule
Answer verification, citation \citep{dhuliawala2023cove,gao2023alce} & nothing; claims rechecked afterward & not chosen \\
Conformal factuality \citep{mohri2024conformal} & claims, at calibrated risk & not chosen \\
Action rules \citep{wang2026agentspec,doshi2026safe} & agent actions & not chosen \\
Evidence-gated control \citep{researchloop2026,prooforstop2026,egvar2026} & claims or transitions & repair or stop \\
Evidence acquisition \citep{zhang2026glean,liu2026hesp} & stopping rule or verifier acceptance & uncertainty or information gain \\
PRIME & graded conclusions, every policy & open route of the strongest reachable conclusion \\
\bottomrule
\end{tabular}
\caption{What each family enforces and how it chooses further evidence. Among these families, only PRIME uses alternative routes both to authorize a conclusion and to choose the next check.}
\label{tab:positioning}
\end{table}

\paragraph{What is new.} PRIME does not claim to originate evidence gating, precommitment, or cost-aware probing. Evidence-gated systems already admit claims only on attested evidence under contracts fixed in advance \citep{researchloop2026,prooforstop2026,egvar2026}, and HESP combines cost-aware probing with evidence-backed stopping \citep{liu2026hesp}. To our knowledge, PRIME is the first to use one committed set of alternative routes for two jobs: deciding which \emph{graded} conclusion the recorded checks authorize, and finding which \emph{still-live} route could authorize a stronger one within budget (Table~\ref{tab:positioning}). The ingredients are classical. Minimal routes are prime implicants \citep{quine1952}, as in diagnosis and truth maintenance \citep{reiter1987,dekleer1986atms}, and budgeted checking is sequential testing \citep{unluyurt2004sequential}. Appendix~\ref{app:related} discusses related work at length.

\paragraph{Advantages.} PRIME uses the same fixed evidence routes to guide checking and authorize conclusions. Every reported conclusion can be audited against its program and evidence ledger.

\paragraph{Costs.} PRIME requires a review program, supplied or generated. Its guarantees concern authorization, not the correctness of the rules or observations (Theorem~\ref{thm:decomposition}). Verifying model judgments adds model calls. With known fields and deterministic checks, scripted checking is a viable alternative; agent-directed checking is relevant when evidence must first be located.

\section{SUFFICE}
\label{sec:bench}

\paragraph{Motivation.} $\tau$-bench scores final database outcomes \citep{yao2024taubench}, Partial Evidence Bench addresses inaccessible evidence \citep{tallam2026partial}, and the Evidence Sufficiency Benchmark tests answering or abstaining with supplied evidence \citep{zhang2026sufficiency}. None jointly evaluates budgeted evidence collection and graded conclusions authorized by completed alternative routes. SUFFICE measures both failures of Section~\ref{sec:intro}, unsupported and missed conclusions.

\paragraph{Construction.} SUFFICE contains 938 tasks built from public records in the Open Source Vulnerabilities (OSV) database, ClinicalTrials.gov registrations, openFDA recalls, and financial filings in the U.S. Securities and Exchange Commission's Electronic Data Gathering, Analysis, and Retrieval (EDGAR) system. A common structure supports comparisons across record schemas. Each of 120 task families fixes a request, three strength-ordered conclusions, routes, and field predicates. The generator evaluates predicates on records to compute the intended standard $f^\star_x$ and checks that resolvers reproduce those values. To block routes, the generator sets required values that records do not meet, occasionally synthetic values outside a field's valid range, and withholds some evidence (Appendix~\ref{app:setup}). An LLM coding assistant drafted candidate fields and conclusion descriptions, but checks are executable.

\paragraph{Splits.} The 938 tasks comprise 480 balanced, 218 natural-prevalence, 192 development, and 48 validation tasks, and development and validation use 60 families disjoint from the held-out ones. The balanced split has 120 tasks per domain, 396 with at least one supported conclusion, and equal numbers of nested tasks, whose successive conclusions use extending routes, and branching tasks, where one conclusion has two alternative sufficient routes. The natural-prevalence split keeps each record's own evidence without constructed failures and is not analyzed here. Two held-out variants are counted separately, a locator-uncertain copy of the balanced split with three candidate fields per atom and a 240-task split that replaces selected field checks with text judgments. Every task leaves $\Phi_x$ empty.

\paragraph{Scoring.} \emph{False authorization} is the share of evaluated episodes reporting a conclusion unsupported by source records. \emph{Premature conclusions}, counted the same way, lack a recorded complete route to themselves or a stronger conclusion, even if factually true. False authorization is judged against the source records, and premature reporting against the recorded checks. \emph{Supported utility} scores the strongest source-supported conclusion reported as 33.3, 66.7, or 100, and zero otherwise. \emph{Emission} is the share of answerable tasks, those supporting some conclusion, that receive one. A direct-answer baseline using Qwen2.5-32B-Instruct without tools scores 30.3 utility and 42.1\% false authorization on the balanced split, showing why utility alone does not establish checked support. Executable requirements, not an LLM judge, determine all scores.

\paragraph{Statistics.} The information and guidance comparisons of Sections~\ref{sec:exp:ladder} and~\ref{sec:exp:guidance} on the balanced split were registered before confirmatory evaluation. We use task families for inference, paired sign-flip tests, 95\% bootstrap confidence intervals (CIs), and Holm correction within registered groups. Other analyses, including premature-conclusion rates added afterward, are secondary. Tables report pooled rates. The corpus, scorer, generator, runtime, agents, and logs are public at the links in Section~\ref{sec:conclusion}. Appendix~\ref{app:setup} details construction and splits, and Appendix~\ref{app:protocol} details statistical procedures.

\section{Experiments}
\label{sec:experiments}

We ask which parts of PRIME prevent unsupported reports and which help agents reach stronger conclusions when evidence is costly or uncertain. Unless noted, experiments use Qwen2.5-32B-Instruct \citep{qwen2024qwen25} on the 480-task balanced split with a budget of 20 tool requests, which include reporting requests. Total model calls and tokens are not matched. Appendix~\ref{app:setup} gives the complete charging rule.

\subsection{Information versus enforcement}
\label{sec:exp:ladder}

\begin{table}[t]
\centering
\fontsize{8}{9.5}\selectfont
\setlength{\tabcolsep}{2.6pt}
\begin{tabular}{lrrrrr}
\toprule
Agent & \makecell[r]{False\\auth.\,\%\,$\downarrow$} & \makecell[r]{Prem.\\\%\,$\downarrow$} & \makecell[r]{Against\\ledger\,$\downarrow$} & \makecell[r]{Util.\\$\uparrow$} & \makecell[r]{Emit\\\%\,$\uparrow$} \\
\midrule
ReAct & 13.0 & 56.9 & 6 & 17.9 & 59.0 \\
\quad +verify & 19.9 & 47.6 & 3 & 10.0 & 49.6 \\
\quad +spec+verify & 22.3 & 33.8 & 55 & 26.9 & 89.9 \\
\quad +program+verify & 22.3 & 25.2 & 56 & 32.7 & 98.5 \\
Fully informed agent & 26.7 & 28.3 & 67 & 38.0 & 99.0 \\
PRIME, no guidance & 0.0 & 0.0 & 0 & 30.1 & 99.0 \\
PRIME & 0.0 & 0.0 & 0 & 45.5 & 100.0 \\
\bottomrule
\end{tabular}
\caption{Information versus enforcement (480 tasks). +verify adds an instruction to verify, +spec the specification view (each atom and its field), +program the routes, and the fully informed agent its ledger. Up to four incomplete episodes were excluded per row. Against ledger counts claims with a recorded failed check on every route. Prem.\ is premature reporting, and Emit is reporting among answerable tasks.}
\label{tab:main}
\end{table}

Does showing an agent the requirements and its checked evidence prevent unsupported reporting, or is external enforcement necessary? We progressively reveal verification instructions, checks, routes, and the ledger, then add PRIME's authorization gate and guidance.

\textbf{More information did not stop unsupported reporting, but the authorization gate did.} Table~\ref{tab:main} shows false authorization rising from 13.0\% for ReAct \citep{yao2023react} to 26.7\% for the fully informed agent, which sees the routes and its own ledger but is not enforced and reported 67 conclusions against that ledger. Adding the gate reduced false authorization to 0.0\%, with or without guidance. On SUFFICE, Theorems~\ref{thm:safety} and~\ref{thm:decomposition} guarantee this zero, because the program is the reference standard and the resolvers reproduce its atom values.

\begin{table}[t]
\centering
\fontsize{8}{9.5}\selectfont
\setlength{\tabcolsep}{2.4pt}
\begin{tabular}{@{}lrrrrr@{}}
\toprule
& \multicolumn{2}{c}{Fully informed} & \makecell[r]{FA\\removed} & \multicolumn{2}{c}{Util.\ with gate\,$\uparrow$} \\
\cmidrule(lr){2-3}\cmidrule(lr){5-6}
Model & FA\,\%\,$\downarrow$ & Util.\,$\uparrow$ & [95\% CI] & no guid. & PRIME \\
\midrule
Qwen2.5-32B & 26.7 & 38.0 & 28.0 [22.7, 33.5] & 30.1 & 45.5 \\
gpt-oss-20b & 17.7 & 30.5 & 18.0 [13.5, 22.7] & 26.9 & 58.5 \\
Mistral-24B & 19.4 & 35.2 & 19.5 [14.9, 24.2] & 29.3 & 36.6 \\
Qwen2.5-7B & 0.4 & 1.0 & 0.5 [0.0, 2.1] & 24.7 & 26.1 \\
\bottomrule
\end{tabular}
\caption{Other models (480 tasks, 20 calls). Mistral-24B is Mistral-Small-24B. FA is false authorization, and FA removed is the family-mean reduction from the gate. With the gate, every model had 0\% false authorization.}
\label{app:tab:transfer}
\end{table}

\textbf{The result held across models.} Table~\ref{app:tab:transfer} shows that the gate removed 18.0 and 19.5 points of false authorization for gpt-oss-20b \citep{openai2025gptoss} and Mistral-Small-24B \citep{mistral2025small3}. Guidance raised utility for every model, by 1.4 to 31.6 points. Qwen2.5-7B shows where the gate has little to remove: without it, the agent stated a conclusion on only 3.8\% of answerable tasks. Under sampled decoding, false authorization was 25.4--26.9\% without the gate and 0\% with it (Appendix~\ref{app:res:a}).

\subsection{Guidance and evidence location}
\label{sec:exp:guidance}
\label{sec:exp:locator}

Does the authorization gate cause agents to stop before gathering enough evidence for stronger conclusions? We test whether guidance and completion recover those conclusions, and whether agent-directed search beats fixed-order checking (blind enumeration) when each condition has three possible evidence locations.

\begin{figure}[tb]
\centering
\includegraphics[width=\columnwidth]{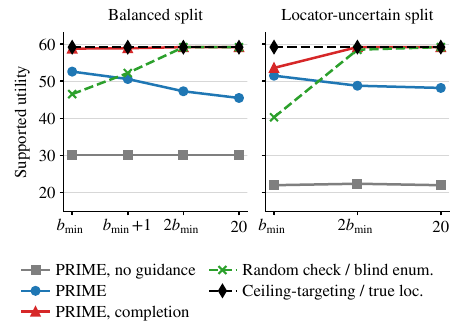}
\caption{Supported utility by budget with known (left) or uncertain (right) evidence locations. Dashed curves denote scripted checking policies (Appendix~\ref{app:setup}). Right-panel completion replays logged episodes without model calls. $b_{\min}$ is the cost of the cheapest committed route to the strongest conclusion before checking.}
\label{fig:budget}
\end{figure}

\textbf{With the gate alone, agents stopped before obtaining stronger supported conclusions.} Figure~\ref{fig:budget} (left) shows utility of 30.1 without guidance and 45.5 with guidance at 20 tool requests. Completion reached 59.2, the maximum for this split, by $2b_{\min}$. Preventing unsupported reports therefore does not ensure that useful remaining checks are performed.

\textbf{When evidence fields were uncertain, choosing where to check helped under tight budgets.} Figure~\ref{fig:budget} (right) shows that at $b_{\min}$ PRIME scored 51.5 against 40.3 for blind enumeration, while a policy given the correct fields (true locators) scored 59.2. At $2b_{\min}$, blind enumeration reached 58.5 but PRIME fell to 48.8 after early stopping, while replayed completion restored 59.2. The benefit of agent-directed search therefore depends on the budget and on the model (Appendix~\ref{app:res:c}).

\subsection{Judging textual evidence}
\label{sec:exp:text}

Can the authorization gate prevent unsupported reporting when an LLM, rather than a deterministic resolver, judges the evidence? We replace selected checks with text judgments on 240 tasks, scored against structured registry fields withheld from the agent, and test requiring a second, agreeing judgment before accepting an observation.

\begin{table}[t]
\centering
\fontsize{8}{9.5}\selectfont
\setlength{\tabcolsep}{2.5pt}
\begin{tabular}{lrrrrr}
\toprule
Agent & \makecell[r]{FA\\\%\,$\downarrow$} & \makecell[r]{Util.\\$\uparrow$} & \makecell[r]{Adm.\\att.} & \makecell[r]{Res.\\err.\,\%\,$\downarrow$} & \makecell[r]{Ep.\\err.\,\%\,$\downarrow$} \\
\midrule
Fully informed agent & 4.2 & 58.5 & 210 & 9.5 & 8.3 \\
PRIME, no guidance & 1.7 & 47.4 & 151 & 13.2 & 8.3 \\
PRIME & 4.6 & 57.5 & 298 & 10.1 & 11.7 \\
PRIME, verified & 0.8 & 55.1 & 222 & 2.7 & 2.5 \\
\bottomrule
\end{tabular}
\caption{Text judgments (240 tasks). FA is false authorization, Adm.\ att.\ counts accepted judgments, Res.\ err.\ is the share of them that are wrong, and Ep.\ err.\ is the share of episodes with such an error.}
\label{tab:text}
\end{table}

\textbf{The gate cannot correct a wrong judgment once the runtime accepts it as evidence.} Table~\ref{tab:text} shows false authorization of 4.6\% for PRIME, falling to 0.8\% with verification, while supported utility decreased from 57.5 to 55.1 and verification required 1.2 additional model calls per episode, outside the tool-request budget. Guidance improved utility but increased false authorization because it prompted more judgments. Runtime authorization thus depends on the accuracy of the recorded evidence.

\subsection{Closest control systems}
\label{sec:exp:extsys}

\begin{table}[t]
\centering
\fontsize{8}{9.5}\selectfont
\setlength{\tabcolsep}{1.9pt}
\begin{tabular}{@{}lrrrrr@{}}
\toprule
System & \makecell[r]{False\\auth.\,\%\,$\downarrow$} & \makecell[r]{Prem.\\\%\,$\downarrow$} & \makecell[r]{Util.\\$\uparrow$} & \makecell[r]{Emit\\\%\,$\uparrow$} & \makecell[r]{Reads\\$\downarrow$} \\
\midrule
Fully informed agent & 26.7 & 28.3 & 38.0 & 99.0 & 3.1 \\
\midrule
AgentSpec \citeyearpar{wang2026agentspec} \\
\quad stop & 0.0 & 0.0 & 38.7 & 87.4 & 3.0 \\
\quad self-examine & 0.0 & 0.0 & 40.1 & 92.9 & 3.1 \\
ResearchLoop \citeyearpar{researchloop2026} & 0.0 & 0.0 & 47.7 & 97.0 & 3.2 \\
Proof-or-Stop \citeyearpar{prooforstop2026} & 0.0 & 0.0 & 52.0 & 99.2 & 4.2 \\
GLEAN \citeyearpar{zhang2026glean} & 14.4 & 14.6 & 35.0 & 83.3 & 3.1 \\
HESP \citeyearpar{liu2026hesp} & 0.0 & 11.7 & 59.2 (49.0) & 100.0 & 3.9 \\
\midrule
PRIME & 0.0 & 0.0 & 45.5 & 100.0 & 2.8 \\
PRIME with completion & 0.0 & 0.0 & 59.2 (59.2) & 100.0 & 3.6 \\
\bottomrule
\end{tabular}
\caption{Adapted systems, sharing tools and requirements, on 480 tasks. Parentheses in Util.\ show scores supported by each agent's recorded checks (HESP and PRIME with completion). Reads counts fields per task.}
\label{tab:extsys}
\end{table}

Does PRIME improve on existing verification and control systems given the same evidence requirements? We compare PRIME with adaptations of AgentSpec \citep{wang2026agentspec}, ResearchLoop \citep{researchloop2026}, Proof-or-Stop \citep{prooforstop2026}, GLEAN \citep{zhang2026glean}, and HESP \citep{liu2026hesp}, reimplemented on PRIME's tool loop because most have no usable public code, with GLEAN applied as a verifier to the fully informed agent's logged runs. They are not exact reruns of the original systems, and they share the model, tools, requirements, resolvers, and budget but not the number of model calls.

\textbf{Some controls prevented unsupported conclusions but sacrificed supported utility.} Table~\ref{tab:extsys} shows zero false authorization for PRIME and every adapted system except GLEAN (14.4\%). AgentSpec, ResearchLoop, and Proof-or-Stop reached utility of 38.7--52.0, against 45.5 for PRIME and 59.2 for PRIME with completion. HESP also scored 59.2, but its conclusion was premature on 11.7\% of tasks. When only conclusions established by recorded checks receive credit, HESP scored 49.0 versus 59.2 for PRIME with completion.

Appendix~\ref{app:extsys} details the adapted controllers and other budgets.

\subsection{Requirements generated from a regulation}
\label{sec:exp:regulation}

\begin{table}[t]
\centering
\footnotesize
\setlength{\tabcolsep}{2pt}
\begin{tabular}{lrrrr}
\toprule
Policy & $n$ & \makecell[r]{Unsupported \%\\{}[95\% CI]\,$\downarrow$} & \makecell[r]{Util.\\$\uparrow$} & \makecell[r]{Cov.\\\%\,$\uparrow$} \\
\midrule
\multicolumn{5}{l}{\emph{Assistant-built program} (526 instances)} \\
\quad Ceiling-targeting & 1,999 & 1.8 [1.3, 2.5] & 18.4 & 26.3 \\
\quad Ceiling-targeting & 600 & 1.7 [0.9, 3.0] & 19.0 & 26.8 \\
\quad Fully informed agent & 600 & 0.2 [0.0, 0.9] & 13.1 & 26.8 \\
\quad PRIME & 600 & 0.2 [0.0, 0.9] & 9.9 & 26.8 \\
\quad PRIME, completion & 600 & 1.7 [0.9, 3.0] & 19.0 & 26.8 \\
\multicolumn{5}{l}{\emph{Protocolized program} (1,999 instances)} \\
\quad Ceiling-targeting & 1,999 & 3.5 [2.8, 4.4] & 47.5 & 92.4 \\
\quad Ceiling-targeting & 600 & 3.2 [2.0, 4.9] & 47.7 & 92.5 \\
\quad Fully informed agent & 600 & 0.0 [0.0, 0.6] & 31.2 & 92.5 \\
\quad PRIME & 600 & 0.3 [0.1, 1.2] & 32.3 & 92.5 \\
\quad PRIME, completion & 600 & 3.2 [2.0, 4.9] & 47.7 & 92.5 \\
\bottomrule
\end{tabular}
\caption{Held-out regulatory evaluation. Cov.\ is the share of registrations with a stated conclusion. Tracker levels are 3 for late, 2 for due, and 1 otherwise. Scripted checks use 1,999 registrations, agents 600, and missing instances count as abstentions. Full-evidence agreement is 22.8\% (assistant-built) and 81.6\% (protocolized).}
\label{app:tab:specbuild}
\end{table}

Can PRIME run requirements that were not written for the benchmark? We generated programs from the U.S. clinical-trial results-reporting regulation in the Code of Federal Regulations (CFR), 42 CFR Part~11, and applied them to ClinicalTrials.gov registrations. A program can conclude that a trial must report results ($c_1$), that they are due ($c_2$), or that they were submitted late ($c_3$). For the \emph{protocolized} program, Qwen2.5-32B-Instruct wrote candidates from the regulation text, and one was selected on 600 validation registrations and frozen. An LLM coding assistant wrote the \emph{assistant-built} program. The reference is an independent reporting-status tracker \citep{devito2018trialstracker,devito2020compliance} that reads the same registry, and a conclusion is unsupported when it is stronger than the tracker's verdict. We test both programs on 1,999 held-out registrations from other sponsors, with agents on 600 of them (Appendix~\ref{app:fdaaa}).

\textbf{Generated programs ran unchanged, and their errors came mainly from their rules.} Table~\ref{app:tab:specbuild} shows that PRIME with completion, using the protocolized program, stated a conclusion for 92.5\% of registrations with 3.2\% unsupported, matching the scripted ceiling-targeting policy (Appendix~\ref{app:setup}). Without completion, the agent stopped at weaker conclusions (utility 32.3 against 47.7), as on the benchmark. Consistent with Theorem~\ref{thm:decomposition}, most unsupported conclusions trace to the program. It omits the regulation's deadline extensions for filed certifications, and 55 of the 70 unsupported conclusions on the 1,999 registrations involve such a certification. The assistant-built program agreed with the tracker on 98.5\% of the registrations it classified, but it classified only 26.3\%.

\section{Conclusion}
\label{sec:conclusion}

PRIME treats evidence-based review as budgeted evidence acquisition under a hard authorization constraint. One committed set of evidence routes both gates what the agent may report and directs it to the checks that could justify a stronger conclusion. The gate holds for every agent policy, and the remaining error comes only from the rules and the recorded observations. With correct rules and exact checks at known locations, completion reaches the strongest supported conclusion once the budget covers the checks on the relevant routes. Empirically, PRIME eliminated false authorization with supplied programs and deterministic checks, guidance and completion recovered the conclusions that the gate alone missed, and model-judged checks showed why the accuracy of admitted observations matters.\label{end:main}

\paragraph{Reproducibility.}
{\raggedright
SUFFICE code and data: \url{https://github.com/HamedKhosravi99/SUFFICE};
PRIME code and episode logs: \url{https://github.com/HamedKhosravi99/PRIME}.\par}

\section*{Limitations}

PRIME's guarantees assume faithful evidence recording and a trusted runtime and renderer. They cover authorization, not the correctness of the requirements or observations (Section~\ref{sec:positioning}). SUFFICE evaluates authorization using executable requirements over four public-record domains, rather than independently validated organizational policies. Confirmatory results use one primary model, Qwen2.5-32B-Instruct, with one greedy run per task, and only this model ran the text-judgment and regulation experiments. Qwen2.5-7B, a registered secondary model, repeated the enforcement and guidance comparisons, as did gpt-oss-20b and Mistral-Small-24B in secondary analyses without multiplicity correction. SUFFICE tasks use empty $\Phi_x$, leaving constraint closure and conflict handling to runtime unit tests. The regulation study compares generated programs with an independent registry tracker rather than legal adjudication. Agent-directed search for where model-judged evidence is located remains untested. Theorem~\ref{thm:ceiling} assumes correct programs, exact checks, and known fields, and free cached checks preclude a general termination guarantee.

\section*{Ethical Considerations}

PRIME targets a risk of using LLM agents for record review, reporting conclusions that their checks do not justify. It refuses such conclusions, and every conclusion it reports can be audited against its program and evidence ledger. Its guarantees concern authorization, so an incorrect rule or observation can still produce an incorrect conclusion (Theorem~\ref{thm:decomposition}). Deployments in clinical, security, financial, or regulatory settings should therefore validate their programs and keep a person responsible for consequential decisions. The regulation study measures agreement with a public tracker and makes no legal determination about any trial or sponsor. All records are public, and Appendix~\ref{app:setup} describes the personal details that some of them contain.

\bibliography{prime_naacl,prime_intro_new}

\clearpage
\appendix
\paragraph{Organization of the appendix.} Appendices~\ref{app:formal} and~\ref{app:proofs} support the formulation, method, and guarantees of Sections~\ref{sec:problem} to~\ref{sec:guarantees}. Appendix~\ref{app:formal} gives the full interaction model that Section~\ref{sec:problem} simplifies, the premises T1--T4 behind the guarantees, program validation, and conformance tests of the runtime and renderer. Appendix~\ref{app:proofs} states PRIME as Algorithm~\ref{alg:prime}, proves Theorems~\ref{thm:safety} to~\ref{thm:ceiling}, and describes a machine check of Theorem~\ref{thm:ceiling} against the runtime's code. Appendices~\ref{app:setup} to~\ref{app:results} support the benchmark and experiments of Sections~\ref{sec:bench} and~\ref{sec:experiments}. Appendix~\ref{app:setup} describes the records and splits, the running example, the agents and models, the charging rule, and the scripted policies. Appendix~\ref{app:protocol} gives the statistical protocol. Appendix~\ref{app:results} reports additional results in the order of Section~\ref{sec:experiments}: the information ladder (\ref{app:res:a}), guidance (\ref{app:res:b}), field choice (\ref{app:res:c}), text judgments (\ref{app:res:d}), and the closest systems (\ref{app:extsys}). Appendix~\ref{app:fdaaa} runs PRIME on programs generated from a clinical-trial regulation. Appendix~\ref{app:pilot} reports an earlier exploratory pilot whose results are not pooled with the others. Appendix~\ref{app:related} extends the comparison of Section~\ref{sec:positioning}.

\section{Formal Details and Implementation Checks}
\label{app:formal}

\paragraph{Intended minimal routes.}
Let $\cP_x^\star(c)$ denote the minimal sufficient routes for conclusion $c$. Formally,
\begin{equation}
\begin{aligned}
\cP_x^\star(c)=\operatorname{Min}_{\Phi_x}\{p\mid{}&\Phi_x\land p\text{ is satisfiable},\\
&p\models_{\Phi_x}f_x^\star(c,\cdot)\}.
\end{aligned}
\label{eq:intended_routes}
\end{equation}
Here, $p\models_{\Phi_x}f_x^\star(c,\cdot)$ means that every assignment satisfying $p$ and $\Phi_x$ justifies $c$. The operator $\operatorname{Min}_{\Phi_x}$ keeps only the routes from which no condition can be removed without losing sufficiency under $\Phi_x$. We assume the intended standard respects the strength order, so $c\preceq_x d$ and $f_x^\star(d,\omega)=1$ imply $f_x^\star(c,\omega)=1$ for every permitted $\omega$.

\paragraph{Full interaction model.} Section~\ref{sec:problem} states the problem in its simplest form. The full model used by the runtime and the proofs is as follows. A case is $x=(S_x,V_x)$, where the \emph{evidence view} $V_x$ holds the record's field values and the agent initially knows only $S_x$. Each atom $e$ has a \emph{resolver} $\sigma_x(e)$, the registered procedure that checks it. A formula $\Phi_x$ over the atoms expresses logical constraints, for example that a reference cannot have a URL unless the reference exists, and $\omega_x$ satisfies $\Phi_x$. The ledger has a fourth state $\stC$ for atoms whose admitted observations disagree. Let $\Cl_{\Phi_x}(z)$ contain the literals established in $z$ together with those $\Phi_x$ implies, excluding conflicting atoms (definition D1 in Appendix~\ref{app:proofs}). Equation~\eqref{eq:authorization} then reads $\Cl_{\Phi_x}(z)\models p$ in place of $z\models p$, and nothing is authorized when the established literals violate $\Phi_x$. With $\Phi_x$ empty and no conflicts, as in every benchmark task, the two definitions coincide. The checking actions $\mathrm{Chk}_x$ comprise field reads, text reads, and attestations, in which a model or person records a judgment as evidence. After history $h_t$, request $a$ costs $\kappa_x(a,h_t)\ge0$. A cached repeat, with matching tool arguments and evidence-read state, is free and returns its recorded result. A check $a_t$ with observation $o_t$ updates the ledger by $z_{t+1}=\Delta_x(z_t,a_t,o_t)$, and observations may be wrong. $\EMIT(c)$ is charged whether or not the runtime accepts it, and $\textsc{stop}$ is free. With remaining budget $b_t$, starting from $b_0=B_x$, the affordable requests are
\begin{equation*}
\begin{aligned}
\cA_t={}&\{a\in\mathrm{Chk}_x\cup\{\EMIT(c)\mid c\in C\}\mid\\
&\quad\kappa_x(a,h_t)\leq b_t\}\cup\{\textsc{stop}\}.
\end{aligned}
\end{equation*}
Other requests are refused without changing the ledger but may be charged. The utility $U(c,x)$ depends only on $c$ and $S_x$, so the runtime can compute it before checking. The planning estimate $\hat\kappa_x(p,z)$ of Equation~\eqref{eq:target} counts a check that settles several literals once, and it may omit the effort of locating a field whose location is uncertain. Every actual check is still charged. The \emph{hold indicator} $\mu$ equals $1$ when $U(c^\star,x)$ exceeds $U(c,x)$ for every $c\in\Auth_{P_0}(z)\cup\{\bot\}$, and $0$ otherwise. PRIME holds weaker reports and issues reminders only while $\mu=1$. The \emph{completion setting} $\gamma\in\{0,1\}$ turns completion on, and when an atom's field is uncertain, completion tries the supplied candidate fields. If checking ends without an accepted request, PRIME reports the set $\MaxAuth_{P_0}(z)$ of maximal authorized conclusions, a single conclusion in the chain setting. The structured result $y(c,z)$ records $c$ and its supporting ledger evidence, and a deterministic renderer $\Render$ turns it into text using fixed templates.

\paragraph{Commitment.} Before checking begins, the runtime validates the program (Program validation, below), rejects an invalid one, records the hash of an accepted one, and only then enables resolvers. Validation tests the program's structure, not whether its rules are correct. In the experiments, each task file supplies the program, the runtime validates it when the task is loaded, and the program cannot change during the episode. The hash-gated commitment and the template renderer are exercised by the conformance tests of Table~\ref{app:tab:conformance}, and the experiment runner scores the structured output directly.

\paragraph{Decision problem.} A \emph{policy} $\pi_\theta$, parameterized by model parameters $\theta$, chooses actions using the specification $S_x$, fixed program $P_0$, and history $h_t$ of previous actions and observations. At termination step $\tau$, the runtime either reports conclusion $c_\tau$ or abstains, written $c_\tau=\bot$. With $f_x^\star(\bot,\omega_x)=0$, the agent seeks a policy maximizing expected utility, with unsupported conclusions receiving no reward.
\begin{equation}\label{eq:decision-problem}
\begin{aligned}
\max_{\pi_\theta}\quad &\mathbb E_{\pi_\theta,\sigma_x}\!\left[U(c_\tau,x)f_x^\star(c_\tau,\omega_x)\right]\\
\mathrm{s.t.}\quad &a_t\in\cA_t,\ b_0=B_x,\ z_0=\stU^m,\\
&b_{t+1}=b_t-\kappa_x(a_t,h_t),\\
&h_{t+1}=h_t\cdot(a_t,o_t),\\
&z_{t+1}=\Delta_x(z_t,a_t,o_t)\ \text{if }a_t\in\mathrm{Chk}_x,\\
&z_{t+1}=z_t\ \text{otherwise},\\
&c_\tau\in\Auth_{P_0}(z_\tau)\cup\{\bot\}.
\end{aligned}
\end{equation}
The expectation covers policy choices and resolver observations. The constraints limit actions to affordable requests and require recorded evidence to authorize the reported conclusion. Without authorization, the agent may report unsupported conclusions.

\paragraph{Partial orders.} The decision problem above writes a single stated conclusion $c_\tau$, which is exact when $\preceq_x$ is a chain. For partially ordered conclusions, an extension of the main formulation, $\MaxAuth_{P_0}(z)$ is an antichain. The runtime's final output $y(\MaxAuth_{P_0}(z),z)$ lists every maximal authorized conclusion with its established route, the renderer states each as a separate finding, the last constraint of~\eqref{eq:decision-problem} is read as $c\in\Auth_{P_0}(z_\tau)$ for each stated $c$, and utility is the maximum over the stated set. Theorems~\ref{thm:safety} and~\ref{thm:decomposition} hold in this form, since neither proof compares distinct maximal elements, and Theorem~\ref{thm:ceiling} assumes a chain.

\paragraph{Residual status and conflicts.} The runtime gives each literal of a route exactly one status, which in order of precedence is satisfied, repairable (a conflict or a retractable opposing attestation blocks it), open (a check can establish it), or impossible (a terminal unavailability certificate). Precedence resolves overlaps, and if no status applies, the runtime raises an error instead of guessing. If the settled atoms contradict $\Phi_x$, it authorizes nothing, suspends route statuses, and exposes the contradicting attestations for review, never the constraint formula, so Equations~\eqref{eq:residual} and~\eqref{eq:target} apply to consistent ledgers only. A formula unsatisfiable on its own is rejected before commitment. Neither failure state, no applicable status or settled atoms that contradict $\Phi_x$, arises in the structured splits, where $\Phi_x$ is empty and each atom has one deterministic resolver.

\paragraph{Output and trusted components.} The structured output lists the authorized conclusions, the atoms on their routes, the unresolved stronger conclusions with their residual requirements, the scope, and qualifications. The map $\denote$, which returns the conclusions a rendered text asserts (D4 in Appendix~\ref{app:proofs}), reads only the zone of authorized findings. Unresolved stronger conclusions appear in a separate zone that lists their missing checks and asserts nothing. Each rendered finding uses one of 13 templates whose slots accept identifiers, numbers, dates, or a quoted span of at most 200 characters drawn only from letters, digits, spaces, parentheses, and the characters \texttt{,.;:/\_'-}, so a slot can never forge a finding or a conclusion. The guarantees assume (T1) a finite list of conclusions with a mechanically checkable order, (T2) a constraint formula supplied from outside and frozen for the episode, (T3) output produced only by the sound renderer, and (T4) a ledger changed only by faithfully recording, under the conflict rule of (D1) in Appendix~\ref{app:proofs}, the values the runtime admits from resolvers or identified attesters, or by retracting such a record. Exactly resolved atoms receive the full guarantee. For atoms attested by a person or a model, it covers only a faithful record of the attestation, and their errors enter $\Ev{res}$. Table~\ref{app:tab:conformance} reports the implementation checks of these premises.

\paragraph{Program validation.} Before commitment the runtime rejects programs whose conclusions, atoms, or routes are malformed, whose atoms name no registered resolver or embedded record, whose routes are unsatisfiable together with $\Phi_x$ or strictly contain another route of the same conclusion, or whose routes are not order-coherent, that is, some route of a stronger conclusion contains no route of a weaker comparable one. A check whose field path or predicate cannot be evaluated establishes nothing at run time, and in our builds a separate linter caught such bindings before commitment.

\begin{table*}[t]
\centering
\small
\begin{tabular}{lccccc}
\toprule
Condition & \makecell{Mask\\violations} & \makecell{Conflict\\fail-open} & \makecell{Premature\\conclusion} & \makecell{Renderer\\strengthening} & \makecell{Retraction\\failure} \\
\midrule
PRIME runtime & 0 & 0 & 0 & 0 & 0 \\
\multicolumn{6}{l}{\footnotesize Templates: 52 of 52 round trips exact; strengthening mutations: 9 of 9 rejected; slot injections: 8 of 8 blocked.} \\
\midrule
Commitment input & \multicolumn{5}{l}{Share of the 70 pairs with different committed programs} \\
\midrule
Specification view only (production) & \multicolumn{5}{l}{0.000} \\
Plus each atom's authored description & \multicolumn{5}{l}{0.429} \\
Plus the true values and the answer key & \multicolumn{5}{l}{0.786} \\
\bottomrule
\end{tabular}
\caption{Runtime and renderer tests include 52 round trips, 9 mutations, 8 injections, and 27 renderer cases. Violation counts come from a scripted episode with 3 conclusions, 4 atoms, 1 conflict, and 1 retraction. Commitment tests compare 70 hidden-evidence pairs across 5 seeds and 14 perturbation types. Deliberately leaky inputs confirm detection.}
\label{app:tab:conformance}
\end{table*}

\section{Proofs of the Results in Section~\ref{sec:guarantees}}
\label{app:proofs}

Algorithm~\ref{alg:prime} restates the procedure of Section~\ref{sec:method} as pseudocode; the proofs cite its lines.

\begin{algorithm}[H]
\caption{One PRIME review. Inputs: $S_x$, $B_x$, and completion setting $\gamma$.}
\label{alg:prime}
\footnotesize
\begin{algorithmic}[1]
\State $P_0\gets G(S_x,r)$; validate $P_0$, reject if invalid; record $\mathrm{hash}(P_0)$ and enable checks
\State $z\gets\stU^m$; $b\gets B_x$; $h\gets\emptyset$; $n\gets0$ \Comment{$n$ counts reminders}
\State \textbf{while} $b>0$ \textbf{do}
\State \quad compute $\Auth_{P_0}(z)$; $(c^\star,p^\star)$ by Equation~(\ref{eq:target}); $\mu\gets1$ if $U(c^\star,x)>\max_{c\in\Auth_{P_0}(z)\cup\{\bot\}}U(c,x)$, else $0$ \label{alg:auth}
\State \quad $a\gets\pi_\theta(S_x,P_0,z,h,b,\Auth_{P_0}(z),c^\star,R(p^\star,z))$
\State \quad \textbf{if} $\kappa_x(a,h)>b$ \textbf{then} refuse $a$; end checking \label{alg:afford}
\State \quad $b\gets b-\kappa_x(a,h)$ \label{alg:charge}
\State \quad \textbf{if} $a=\EMIT(c)$ \textbf{then} recompute $(c^\star,p^\star)$ and $\mu$ as in line~\ref{alg:auth} with the charged $b$ \label{alg:remu}
\State \quad \textbf{if} $a\in\mathrm{Chk}_x$ is a registered check with allowed parameters \textbf{then} run its resolver, or return its recorded result if $a$ repeats a cache-eligible request; update $z$, $h$ \label{alg:check}
\State \quad \textbf{else if} $a=\EMIT(c)$, $c\in\Auth_{P_0}(z)$, and $\mu=0$ \textbf{then} \textbf{return} $\Render(y(c,z))$ \label{alg:accept}
\State \quad \textbf{else if} $a=\EMIT(c)$, $c\in\Auth_{P_0}(z)$, and $\mu=1$ \textbf{then} hold the request; update $h$ \label{alg:hold}
\State \quad \textbf{else if} $a=\textsc{stop}$, $\mu=1$, and $n<3$ \textbf{then} remind the agent; $n\gets n+1$; update $h$ \label{alg:remind}
\State \quad \textbf{else if} $a=\textsc{stop}$ \textbf{then} \textbf{if} $\mu=1$ and $\gamma=1$, run the open checks of $p^\star$ while $\mu=1$ and $b>0$, recomputing $(c^\star,p^\star)$ and $\mu$ after each and charging each as in line~\ref{alg:charge}; end checking \label{alg:complete}
\State \quad \textbf{else} refuse $a$; update $h$ without changing $z$ \label{alg:refuse}
\State \textbf{return} $\Render(y(\MaxAuth_{P_0}(z),z))$, a single conclusion in a chain, or abstain if empty \label{alg:final}
\end{algorithmic}
\end{algorithm}

Line~\ref{alg:afford} never applies in our benchmark, because every charge is zero or one call and checking ends when the budget reaches zero. Every emission is charged (line~\ref{alg:charge}), so line~\ref{alg:remu} recomputes the target and the hold test with the budget left after that charge, and an emission is held only while a stronger conclusion remains reachable within that budget. The final statement in line~\ref{alg:final} is free, so an authorized conclusion is still reported when the budget runs out, including one whose emission was held. Cached repeats are free, so a policy can repeat them without reducing $b$, and the loop need not terminate for every policy. The guarantees therefore hold for every episode that produces an output.

Each line of a proof chains steps with arrows. A bracket names the premise, result, or line of Algorithm~\ref{alg:prime} that a step uses; steps that only apply a definition or routine algebra carry no bracket.

\paragraph{Definitions used in the proofs.}
(D1) The update rule $\Delta_x$ sets an atom to $\stT$ ($\stF$) when every admitted, unretracted value for it is 1 (0), to $\stC$ when such values disagree, and leaves it $\stU$ when there is none. For a ledger $z$, the settled literals are $L(z)=\{e_i:z_i=\stT\}\cup\{\neg e_i:z_i=\stF\}$ and the conflicting atoms are $K(z)=\{i:z_i=\stC\}$. If $\Phi_x\wedge\bigwedge L(z)$ is unsatisfiable, then $\Cl_{\Phi_x}(z)=\bot$; otherwise $\Cl_{\Phi_x}(z)=\{\ell:\text{the atom of }\ell\notin K(z),\ \Phi_x\wedge\bigwedge L(z)\models\ell\}$. A route $p$ holds at $z$, written $\Cl_{\Phi_x}(z)\models p$, when $\Cl_{\Phi_x}(z)\neq\bot$ and $p\subseteq\Cl_{\Phi_x}(z)$. We abbreviate $\Cl_{\Phi_x}$ to $\Cl$.
(D2) $\MaxAuth_P(z)=\{c\in\Auth_P(z):\text{no }d\in\Auth_P(z)\text{ has }c\prec d\}\subseteq\Auth_P(z)$, where $c\prec d$ means $c\preceq d$ and $c\neq d$.
(D3) In each pass of the loop of Algorithm~\ref{alg:prime}, the runtime computes $\Auth_{P_0}(z_t)$ (line~\ref{alg:auth}) and decides on the proposed action without changing the ledger in between. A check sets $z_{t+1}=\Delta_x(z_t,a_t,o_t)$ (line~\ref{alg:check}), and completion (line~\ref{alg:complete}) applies $\Delta_x$ once per check it runs, so $z_\tau$ is the ledger after completion. A refused request (line~\ref{alg:refuse}), a held emission (line~\ref{alg:hold}), and a reminder (line~\ref{alg:remind}) only extend the history, so $z_{t+1}=z_t$, and output nothing. Every request except $\textsc{stop}$ and a cached repeat of a check or of an unrecognized request reduces $b$, so emissions are always charged, and the episode ends when an emission is accepted, when checking ends, or when the budget is spent. An emission $\EMIT(c)$ is accepted (line~\ref{alg:accept}) only if $c\in\Auth_{P_0}(z_t)$ and $\mu=0$, and acceptance ends the episode. Otherwise the episode ends at a time $\tau$ and the runtime states $\MaxAuth_{P_0}(z_\tau)$ or abstains (line~\ref{alg:final}).
(D4) When an episode ends at ledger $z$, the runtime returns a structured output $y$; $M(y)$ denotes the conclusions $y$ marks as authorized: the stated conclusion and the authorized conclusions $y$ lists, which the runtime takes from $\Auth_{P_0}(z)$. ``PRIME outputs $c$ at $z$'' means $c\in M(y)$. The user sees $s=\Render(y)$ and the conclusions $\denote(s)$.
(D5) $\Ev{mask}=\{M(y)\not\subseteq\Auth_{P_0}(z)\}$, $\Ev{render}=\{\denote(s)\not\subseteq M(y)\}$, and $\Ev{user}=\{\exists c\in\denote(s):f^\star_x(c,\omega_x)=0\}$; $\Ev{prog}$ and $\Ev{res}$ are as in Section~\ref{sec:guarantees}.

Part (i) of Theorem~\ref{thm:safety} in full reads as follows. Let $x=(S,V)$ and $x'=(S,V')$ have the same specification and arbitrary evidence views. If $G$ is invoked only on $S$ and fixed randomness $r$, and no resolver can run before the hash of the committed program is recorded, both cases receive the program $G(S,r)$ with the same hash. If $G$ draws its seed for both cases from one distribution that does not depend on the evidence view, the committed programs have the same distribution.

\begin{proof}[Proof of Theorem~\ref{thm:safety}, part (i)]
Model $G$ as an algorithm that runs in rounds $j=1,2,\dots$. In round $j$, if $\delta_j(S,r,a_1,\dots,a_{j-1})=\mathrm{halt}$, it returns $g_j(S,r,a_1,\dots,a_{j-1})$, and otherwise it sends a resolver request and receives $a_j$. It halts after finitely many rounds. The hash is computed from the program it returns, so every request precedes the recording of the hash, and the gate answers it with $\bot$ whatever the evidence view. Hence, by induction on $j$,
\begin{flalign*}
&a_j=\bot\ \ \forall j&\\
&\Rightarrow J=\min\{j:\delta_j(S,r,\bot^{j-1})=\mathrm{halt}\}&\\
&\Rightarrow P_0=g_J(S,r,\bot^{J-1})=:\hat G(S,r)&\\
&\Rightarrow P_0(x)=\hat G(S,r)=P_0(x')\by{$S_x=S_{x'}$}&\\
&\Rightarrow\mathrm{hash}(P_0(x))=\mathrm{hash}(P_0(x')).&
\end{flalign*}
Neither evidence view appears in these lines, so both are arbitrary. If $G$ is randomized and the cases draw seeds $\rho,\rho'$ from the same distribution $\nu$, which does not depend on the evidence view, then for every set $\mathcal A$ of programs
\begin{flalign*}
&\Pr(P_0(x)\in\mathcal A)=\Pr(\hat G(S,\rho)\in\mathcal A)&\\
&=\nu(\{r:\hat G(S,r)\in\mathcal A\})\by{$\rho\sim\nu$}&\\
&=\Pr(\hat G(S,\rho')\in\mathcal A)\by{$\rho'\sim\nu$}&\\
&=\Pr(P_0(x')\in\mathcal A).&
\end{flalign*}
\end{proof}

\begin{proof}[Proof of Theorem~\ref{thm:safety}, part (ii)]
By (D3) the runtime outputs conclusions only when the episode ends, either with an accepted $\EMIT(c)$ at time $t$ ($z=z_t$) or at a time $\tau$ without one ($z=z_\tau$). Let $L\subseteq\Auth_{P_0}(z)$ be the authorized conclusions the output lists (D4).
\begin{flalign*}
&\text{Case 1: }\EMIT(c)\text{ accepted}&\\
&\Rightarrow c\in\Auth_{P_0}(z_t)\by{lines~\ref{alg:auth},~\ref{alg:accept}}&\\
&\Rightarrow M(y)=\{c\}\cup L\subseteq\Auth_{P_0}(z_t);&\\
&\text{Case 2: }M(y)=\MaxAuth_{P_0}(z_\tau)\cup L&\\
&\qquad\subseteq\Auth_{P_0}(z_\tau)\by{line~\ref{alg:final}; D2}&\\
&\Rightarrow M(y)\subseteq\Auth_{P_0}(z)\text{ always}&\\
&\Rightarrow\Ev{mask}=\varnothing.&
\end{flalign*}
If nothing is authorized at $z_\tau$, the runtime abstains and $M(y)=L=\varnothing$. Neither case used a property of the policy, the actions, the resolver outcomes, the budget, or the stopping time.
\end{proof}

\begin{proof}[Proof of Theorem~\ref{thm:safety}, part (iii)]
For any episode with final ledger $z$, output $y$, and text $s=\Render(y)$,
\begin{flalign*}
&\denote(s)\subseteq M(y)\subseteq\Auth_{P_0}(z)&\\
&\qquad\by{soundness; part (ii)}&\\
&\Rightarrow\denote(s)\subseteq\Auth_{P_0}(z);&\\
&\denote(\Render(y))\subseteq M(y)\ \forall y\by{soundness}&\\
&\Rightarrow\Ev{render}=\varnothing.&
\end{flalign*}
Every evaluative sentence the user sees comes from $s$, so every conclusion the user sees lies in $\Auth_{P_0}(z)$.
\end{proof}

\begin{proof}[Proof of Theorem~\ref{thm:decomposition}]
Let $\mathcal G=\Ev{render}\cup\Ev{mask}\cup\Ev{res}$ and take an episode in $\Ev{user}\setminus\mathcal G$, with final ledger $z$, output $y$, and text $s$:
\begin{flalign*}
&\exists c\in\denote(s):f^\star_x(c,\omega_x)=0&\\
&\Rightarrow c\in M(y)\by{$\notin\Ev{render}$}&\\
&\Rightarrow c\in\Auth_{P_0}(z)\by{$\notin\Ev{mask}$}&\\
&\Rightarrow z\text{ truthful},\ c\in\Auth_{P_0}(z),\ f^\star_x(c,\omega_x)=0&\\
&\qquad\by{$\notin\Ev{res}$; T4}&\\
&\Rightarrow\text{the episode lies in }\Ev{prog}.&
\end{flalign*}
Hence $\Ev{user}\subseteq\Ev{prog}\cup\Ev{res}\cup\Ev{mask}\cup\Ev{render}$, the union bound gives~\eqref{eq:decomp}, and parts (ii) and (iii) of Theorem~\ref{thm:safety} make the last two terms zero. On the complement of $\Ev{res}$, every admitted value agrees with $\omega_x$, so no atom becomes conflicting, and because T4, the faithful-ledger assumption of Appendix~\ref{app:formal}, makes the ledger record exactly the admitted values, every settled atom of the final ledger agrees with $\omega_x$.
\end{proof}

\paragraph{Premises of Theorem~\ref{thm:ceiling}.} The conclusions form a chain and $\Phi_x$ is empty. The program is correct in the sense of Section~\ref{sec:problem}. Checks fall into \emph{call classes}: one call of a class settles every atom of the class, and it returns their values under $\omega_x$ (exact resolvers, known fields). Hence every ledger in the episode is truthful and has no conflicts. Completion (line~\ref{alg:complete} of Algorithm~\ref{alg:prime}) repeats the following step while $\mu=1$ and $b>0$: check an atom of $R(p^\star,z)$, charge one call, and recompute $(c^\star,p^\star)$ and $\mu$. $N(z)$ is the number of call classes that contain an atom of $E^\dagger$ that is unknown at $z$.

\begin{proof}[Proof of Theorem~\ref{thm:ceiling}]
Write $\omega_x\models p$ when every literal of $p$ is true under $\omega_x$. Truthfulness gives $z\models p\Rightarrow\omega_x\models p$, and it also gives that a route with $\omega_x\models p$ is live at every ledger.

\emph{Step 1 (nothing above $c^\dagger$ is ever authorized).} For a truthful $z$,
\begin{flalign*}
&c\in\Auth_{P_0}(z)&\\
&\Rightarrow\exists d\succeq c,\ p\in P_0(d):\ z\models p&\\
&\Rightarrow\omega_x\models p\by{truthful}&\\
&\Rightarrow f^\star_x(d,\omega_x)=1\by{correct $P_0$}&\\
&\Rightarrow c\preceq d\preceq c^\dagger\by{definition of $c^\dagger$}.&
\end{flalign*}
If $c^\dagger=\bot$, then $\Auth_{P_0}(z)=\varnothing$ at every ledger, PRIME abstains, and the claim holds. From here on, $c^\dagger\neq\bot$.

\emph{Step 2 (a true route of $c^\dagger$).} By correctness, $f^\star_x(c^\dagger,\omega_x)=1$ gives $d\succeq c^\dagger$ and $p^\dagger\in P_0(d)$ with $\omega_x\models p^\dagger$. Then $f^\star_x(d,\omega_x)=1$, so $d\preceq c^\dagger$ and $d=c^\dagger$. Thus $p^\dagger\in P_0(c^\dagger)$ is live at every ledger, and its atoms lie in $E^\dagger$.

\emph{Step 3 (invariant $b\ge N(z)$ keeps the target at or above $c^\dagger$).} If $b\ge N(z)$, then
\begin{flalign*}
&\hat\kappa(p^\dagger,z)\le N(z)\le b\by{atoms of $p^\dagger\subseteq E^\dagger$}&\\
&\Rightarrow c^\dagger\in\cR(z,b)\by{Step 2; Eq.~\eqref{eq:target}}&\\
&\Rightarrow U(c^\star)\ge U(c^\dagger)\Rightarrow c^\star\succeq c^\dagger\by{chain}.&
\end{flalign*}

\emph{Step 4 (each completion step preserves the invariant).} Suppose $\mu=1$ and $b\ge N(z)$. Then $c^\star$ is strictly stronger than every authorized conclusion, so $p^\star$ is not complete. It is live, so $R(p^\star,z)\neq\varnothing$ and every literal of $R(p^\star,z)$ has an unknown atom $e$. Because $p^\star\in P_0(c^\star)$ and $c^\star\succeq c^\dagger$ (Step 3), $e\in E^\dagger$. Checking $e$ settles its whole class, which contains an unknown atom of $E^\dagger$, so $N$ falls by exactly one while $b$ falls by one. The invariant holds again, and since $N\ge0$ falls at each step, completion stops after at most $N(z)$ steps.

\emph{Step 5 (the final ledger authorizes $c^\dagger$).} Checking ends at a ledger $z$ with $b\ge N(z)$, and by Step 4 the invariant still holds when completion stops. Completion stops, or is skipped, in one of two cases. If $\mu=0$, some authorized $c$ has $U(c)\ge U(c^\star)$, so $c\succeq c^\star\succeq c^\dagger$ by Step 3, and Step 1 forces $c=c^\dagger$. If $b=0$, then $N(z)=0$, every atom of $p^\dagger$ is settled, truthfully, so $z\models p^\dagger$ and $c^\dagger\in\Auth_{P_0}(z)$. In both cases Step 1 gives $\MaxAuth_{P_0}(z)=\{c^\dagger\}$, and line~\ref{alg:final} reports $c^\dagger$.

For the corollary: a charged agent request that checks an unknown atom of $E^\dagger$ lowers $b$ and $N$ by one each, and cached repeats change neither, so $b\ge N(z)$ holds from $B\ge N(z_0)$ to the end of checking.
\end{proof}

\paragraph{Machine check.} In the tightness example of Section~\ref{sec:guarantees}, each atom is its own call class, $c^\dagger=c_2$, and $E^\dagger=\{e_1,\dots,e_5\}$. Test \texttt{test\_completion\_ceiling.py} in the released runtime checks the theorem against the runtime's own target computation on 500 random programs with up to six atoms and shared call classes. For each program, it covers every truth assignment, an arbitrary prefix of agent checks, and budgets $N(z)$, $N(z)+1$, and $N(z)+3$, and it records instances where $N(z)-1$ fails. It also confirms the example's outcomes at 3, 4, and 5 calls.

\section{Experimental Setup}
\label{app:setup}

This appendix gives the conditions shared by all experiments. Table~\ref{app:tab:overview} lists the studies and the exhibits that report them.

\begin{table*}[t]
\centering
\footnotesize
\setlength{\tabcolsep}{3pt}
\begin{tabular}{>{\raggedright\arraybackslash}p{0.105\textwidth}>{\raggedright\arraybackslash}p{0.205\textwidth}>{\raggedright\arraybackslash}p{0.14\textwidth}>{\raggedright\arraybackslash}p{0.2\textwidth}>{\raggedright\arraybackslash}p{0.13\textwidth}>{\raggedright\arraybackslash}p{0.14\textwidth}}
\toprule
Study & Question & Tasks or records & Agents and policies & What varies & Exhibits \\
\midrule
Information vs.\ enforcement (\S\ref{sec:exp:ladder}) & Does an agent shown the requirements and its ledger follow them? & Balanced split, 480 tasks & Qwen2.5-32B-Instruct; three other models & Information in the prompt; authorization gate & Table~\ref{tab:main}; Figure~\ref{app:fig:ladder}; Table~\ref{app:tab:transfer} \\
Guidance (\S\ref{sec:exp:guidance}) & Does guidance recover the conclusions that enforcement alone misses? & Balanced split & Qwen2.5-32B-Instruct; scripted policies & Guidance, completion, budget & Figure~\ref{fig:budget}, left; Figures~\ref{app:fig:controllaw} and~\ref{app:fig:rejection} \\
Evidence location (\S\ref{sec:exp:locator}) & Does the agent's choice of field beat fixed-order checking? & Locator-uncertain copy & Qwen2.5-32B-Instruct, gpt-oss-20b, Mistral-Small-24B; scripted policies & Field choice, completion, budget & Figure~\ref{fig:budget}, right; Table~\ref{app:tab:multimodel} \\
Text judgments (\S\ref{sec:exp:text}) & What happens when the model judges text? & Text-evidence split, 240 tasks & Qwen2.5-32B-Instruct & Guidance; verification of attestations & Table~\ref{tab:text} \\
Closest systems (\S\ref{sec:exp:extsys}) & What does PRIME add over the closest systems given the same requirements? & Balanced split; locator-uncertain copy & AgentSpec, ResearchLoop, Proof-or-Stop, HESP, GLEAN; Qwen2.5-32B-Instruct & System; budget & Table~\ref{tab:extsys}; Tables~\ref{app:tab:extsys} and~\ref{app:tab:hesp} \\
Regulation (App.~\ref{app:fdaaa}) & Can PRIME execute requirements generated from a published regulation? & 1,999 held-out registrations; 600 for agents & Two generated programs; ceiling-targeting policy; Qwen2.5-32B-Instruct agents & Construction workflow; policy & Tables~\ref{app:tab:fdaaa:attrition} to~\ref{app:tab:fdaaa:confusion} \\
Pilot (exploratory) & How did enforcement behave on earlier hand-built tasks? & 60 hand-built tasks & gpt-4.1, gpt-4o-mini; Qwen2.5-7B and 32B for tests & Enforcement; static checklist & Table~\ref{app:tab:sixdomain} \\
\bottomrule
\end{tabular}
\caption{Study overview for Section~\ref{sec:experiments} and Appendices~\ref{app:fdaaa}--\ref{app:pilot}. Pilot results are not pooled.}
\label{app:tab:overview}
\end{table*}

\begin{table*}[t]
\centering
\small
\setlength{\tabcolsep}{5pt}
\begin{tabular}{llll}
\toprule
Atom & Field & Predicate & Value \\
\midrule
$e_1$ & \texttt{details} & exists & true \\
$e_2$ & \texttt{database\_specific.} & equals \texttt{True-alt} & false \\
      & \texttt{\ \ github\_reviewed} & & \\
$e_3$ & \texttt{database\_specific.severity} & exists & true \\
$e_4$ & \texttt{references.0.type} & equals \texttt{ADVISORY-alt} & false \\
$e_5$ & \texttt{references.0.url} & exists & true \\
$e_6$ & \texttt{database\_specific.cwe\_ids.0} & exists & true \\
$e_7$ & \texttt{database\_specific.} & exists & true \\
      & \texttt{\ \ github\_reviewed} & & \\
$e_8$ & \texttt{affected.0.package.ecosystem} & equals ``PyPI'' & true \\
\bottomrule
\end{tabular}
\caption{Routes to $c_1$, $c_2$, and $c_3$ are $\{e_1\}$, $\{e_1,e_2,e_3\}$ or $\{e_1,e_4,e_5\}$, and $\{e_1,e_2,e_3,e_6,e_7\}$, respectively. Constructed failures at $e_2$ and $e_4$ leave $c_1$ strongest.}
\label{app:tab:example}
\end{table*}

\begin{table*}[t]
\centering
\small
\setlength{\tabcolsep}{4pt}
\begin{tabular}{>{\raggedright\arraybackslash}p{0.22\textwidth}>{\raggedright\arraybackslash}p{0.58\textwidth}cc}
\toprule
Agent & What it sees or does & Enforced & Guided \\
\midrule
ReAct & the request, scope, conclusions, and tools \citep{yao2023react} & no & no \\
\quad +verify & plus an instruction to verify before concluding & no & no \\
\quad +spec & ReAct plus the specification view: each atom and the field it reads & no & no \\
\quad +spec+verify & both & no & no \\
\quad +program & +spec plus the committed routes & no & no \\
\quad +program+verify & both & no & no \\
Fully informed agent & +program+verify plus a live display of its own ledger & no & no \\
PRIME without guidance & the same prompt plus the authorized conclusions; others are refused & yes & no \\
PRIME & plus the strongest reachable conclusion, its open atoms, and reminders & yes & yes \\
\bottomrule
\end{tabular}
\caption{Section~\ref{sec:exp:ladder} information ladder. Agents share model, tools, budget, and scorer.}
\label{app:tab:agents}
\end{table*}

\paragraph{Records and splits.} The records were retrieved on 1 August 2026. GitHub Advisory Database records, which OSV serves, are licensed under CC BY 4.0 and are redistributed with attribution. Records from openFDA are in the public domain, and ClinicalTrials.gov and EDGAR records are public U.S. government information. The sources are OSV (\url{https://osv.dev}), ClinicalTrials.gov (\url{https://clinicaltrials.gov}), openFDA (\url{https://open.fda.gov}), and SEC EDGAR (\url{https://www.sec.gov/edgar}), and the records and task texts are in English. Some advisory texts contain personal details, such as a reporter's or vendor's contact e-mail address. No atom tests such personal details, and the records are kept as published. The SUFFICE data are released under CC BY 4.0, with each embedded record under its source's terms, and the code under the MIT License. The held-out splits share 60 families, and development (192 tasks, 48 families) and validation (48 tasks, 12 families) use the other 60. An LLM coding assistant working with the authors drafted the families' candidate record fields and conclusion sentences in code. Task instances were then generated and verified mechanically. The reference judgment $f^\star_x$ evaluates the authored standard on the stored atom values, never reads an agent's routes or ledger, and scores every agent identically. The balanced split is built so that unsupported conclusions are possible. In 199 of its 480 tasks (41.5\%), an atom on a route has a requirement the record does not meet. These tasks contain 233 such atoms, of which 104 require a value taken from another record, 55 a numeric threshold the record cannot meet, 23 a date a year earlier than the record's, and 51 a synthetic value outside the field's valid values, such as \texttt{True-alt} for a Boolean review flag or \texttt{ADVISORY-alt} for a reference type. In 129 tasks the evidence for some atom is unavailable (146 atoms). Rates on the balanced split therefore describe agent behavior under this mix, not how often unsupported conclusions arise in practice. The benchmark leaves $\Phi_x$ empty even where two atoms read the same field. The locator-uncertain copy keeps every task and answer of the balanced split but withholds each atom's field. Instead, the atom describes what it must read and lists three candidate fields, the true one among them. The text-evidence split has 240 clinical and security tasks, and Appendix~\ref{app:res:d} gives its construction. The released corpus has 1,898 task files, the 938 tasks plus the 480 locator-uncertain copies, the 240 text-evidence tasks, and a 240-task one-branch split that this paper does not analyze.

\paragraph{The running example.} The security example in Sections~\ref{sec:intro} to~\ref{sec:method} and Figure~\ref{fig:overview} is hypothetical. It illustrates the route structure of task \texttt{sec-v2-0114}, generated from OSV advisory \texttt{GHSA-5pqf-rvm7-3wgw}, not that advisory's review status. The task's conclusion sentences are ``The advisory identifies an affected package with a readable summary'' ($c_1$), ``The advisory is triage-ready: affected package identified, summary present, and its review/severity metadata are populated'' ($c_2$), and ``The advisory satisfies the checked currency-and-completeness protocol'' ($c_3$), and the illustrated acceptance decision corresponds to $c_2$. Table~\ref{app:tab:example} gives the actual atoms and routes. Atom $e_1$ reads the record's \texttt{details} field, although the conclusion sentences call it a summary. Atom $e_2$ requires the advisory's GitHub curator-review flag to equal the altered value \texttt{True-alt}, which no record can have, although this advisory's flag is \texttt{true}. Atom $e_4$ likewise requires the first reference's type to equal \texttt{ADVISORY-alt}, although that reference is typed \texttt{ADVISORY}. These are constructed requirement failures, not evidence that the advisory was unreviewed or that its reference is not a published advisory. The other atoms test that the severity ($e_3$) and the first reference's URL ($e_5$) exist, the presence of a weakness identifier ($e_6$) and of the review flag ($e_7$), and whether the ecosystem is PyPI ($e_8$), which lies on no route. In this task the commitment procedure sees each atom's tool, parameters, and field but never its value or the routes.

\paragraph{Agents and models.} Confirmatory runs serve Qwen2.5-32B-Instruct \citep{qwen2024qwen25} with vLLM \citep{kwon2023vllm} at temperature zero, one run per task. The checkpoints are \nolinkurl{Qwen/Qwen2.5-32B-Instruct}, \nolinkurl{Qwen/Qwen2.5-7B-Instruct}, \nolinkurl{openai/gpt-oss-20b} \citep{openai2025gptoss}, and \nolinkurl{mistralai/Mistral-Small-24B-Instruct-2501} \citep{mistral2025small3}, all under the Apache 2.0 license, served by vLLM 0.24.0 (0.10.1 for part of the Mistral-Small-24B runs) with a 16,384-token context. Qwen2.5-32B-Instruct ran on two A100 GPUs for the registered comparisons and on one or two A100 80GB, H100, H200, or RTX Pro Blackwell GPUs for the later experiments. The runs behind the reported results used about 90 GPU-hours, and the development-split runs that set the minimum effect for guidance used about 20 more. Table~\ref{app:tab:agents} lists what each agent of Section~\ref{sec:exp:ladder} sees. PRIME without guidance adds the \emph{mask}, the authorization gate of Section~\ref{sec:exp:ladder}, which refuses conclusions the recorded checks do not authorize. The fully informed agent and PRIME without guidance receive identical prompts except for the list of authorized conclusions, which a regression test verifies. The experiments of Sections~\ref{sec:exp:ladder} and~\ref{sec:exp:guidance} were repeated with gpt-oss-20b, served on H200 GPUs, and with Mistral-Small-24B, served on H100 and A100 80GB GPUs. The enforcement and guidance comparisons were also repeated with Qwen2.5-7B, served on one A100 GPU. The text-judgment experiment and the agent runs of the regulatory evaluation used Qwen2.5-32B-Instruct only. vLLM's Mistral tool-call parser was wrapped to accept a single JSON object where it expects a list (in a diagnostic cell, 29 of 529 requests) and to return as text the few outputs that name a nonexistent tool or fabricate a tool result (19 of 529); well-formed outputs are unaffected. An episode whose model calls still failed after the client's retries carries no measurement and is left out of the rates, which are therefore taken over completed episodes; this removed at most four episodes from any reported cell (eight of the 4,320 episodes of Section~\ref{sec:exp:ladder}), and the released code lists each one.

\paragraph{Budgets, stopping, and resolvers.} Budgets count tool calls, and one call can settle several atoms, as in the running example. Every request costs one call except as follows. Only the first request of each model turn is executed, and the others are ignored at no cost. A request that repeats an executed request, with the same tool and arguments and, for an attestation, the same read status of its source text, is answered from a cache at no cost. An unrecognized request is refused and charged, and an identical repeat of it is answered from the cache like any other repeat. Reporting requests are never answered from the cache, so they are always charged, whether accepted, held, or refused. A turn without a request is free, as are the reminders and the runtime's final statement. Since every charge is zero or one call and checking ends when the budget reaches zero, no request is ever unaffordable. Every reported outcome is the one this charging rule produces.
If an unenforced agent stops without stating a conclusion, it is asked once, at no cost, which conclusions it asserts, and that answer is scored. PRIME's runtime instead states the strongest authorized conclusion. Field atoms are decided by deterministic resolvers that evaluate a predicate on one field of the task's embedded record. Four scripted policies choose their checks by code but use the same resolvers, ledger, and final statement as PRIME. At every step, the \emph{ceiling-targeting} policy takes, among the live routes of all conclusions, the one whose open atoms need the fewest distinct calls, with ties going to the stronger conclusion, and checks the first of those atoms by identifier. It continues until no live route has an open atom or the budget is spent, so it does not stop when a conclusion becomes authorized. The \emph{random-order} policy (random check in Figure~\ref{fig:budget}) checks a random open atom of a live route. On the locator-uncertain copy, \emph{blind enumeration} makes the ceiling-targeting policy's choices but reads an atom's candidate fields in listed order, one call per attempt, and treats the atom as impossible once the list is exhausted, and the \emph{true-locator} policy is ceiling-targeting with the correct fields. Completion on that copy is \emph{replayed}, that is, applied to PRIME's logged episodes from the agent's stopping point without calling the model. Text atoms are judged by the model, which \emph{attests} them with a value and a quote; Appendix~\ref{app:res:d} gives the admission rule and the \emph{verified} condition.

\paragraph{Endpoints.} If an episode states several conclusions, false authorization counts it if any is unsupported, and utility takes the strongest supported one. Both endpoints credit a true conclusion even when the agent did not check it, which favors unenforced agents. The ladder runs nine agents on 480 tasks at 20 calls (4,320 episodes), and the two PRIME agents also run at the three other budgets (2,880 episodes).

\paragraph{Reproducibility and artifacts.} We provide two public repositories for reproducibility: SUFFICE (\url{https://github.com/HamedKhosravi99/SUFFICE}) and PRIME (\url{https://github.com/HamedKhosravi99/PRIME}), whose \texttt{logs/} folder holds the archive of every episode log behind the reported numbers. The SUFFICE repository contains the benchmark generator, task specifications, data splits, executable checks, and scoring procedures. The PRIME repository contains the review runtime, evidence ledger, authorization gate, guidance and completion procedures, agent configurations, and evaluation scripts. The 938-task corpus comprises 480 balanced, 218 natural-prevalence, 192 development, and 48 validation tasks; locator-uncertain and text-evidence variants are evaluated separately. Released experimental records and analysis scripts support inspection of tool requests, recorded evidence, final conclusions, attestation errors, budget accounting, and reported metrics. Reproduction instructions are provided in the repository documentation.

\section{Statistical Protocol}
\label{app:protocol}

Before any confirmatory episode was scored, we froze the utility function, endpoints, agents, budgets, tests, multiplicity groups, and the minimum effect for guidance.

\paragraph{Unit of inference and tests.} The family is the unit of inference, because the tasks of one family share their request, conclusions, routes, and atoms. A contrast is a family-mean paired difference between two agents on the same tasks, tested by a sign-flip randomization test and reported with a 95\% bootstrap interval. An exact McNemar test that treats tasks as independent complements each randomization test, and discordant families are those in which the two compared agents differ. Holm correction applies within three registered groups: the ladder contrasts of Section~\ref{sec:exp:ladder}, the guidance contrasts of Section~\ref{sec:exp:guidance}, and induction-method contrasts that this paper does not report.

\paragraph{Minimum effect for guidance.} The minimum effect was set on development tasks with branching routes as $\delta=\max(5.0,\,0.2\,G_{\mathrm{dev}})$ with $G_{\mathrm{dev}}=30.79$, the utility gap between PRIME without guidance and an informed route planner (an LLM-free policy that knows which routes hold but still pays for their checks), giving $\delta=6.16$. The power of an equivalence test at zero effect was 0.708, below the registered 0.8, so, as decided in advance, no equivalence conclusion is drawn.

\paragraph{Confirmatory and secondary analyses.} The registered confirmatory analysis consists of the ladder and guidance contrasts on the balanced split, with Qwen2.5-32B-Instruct as the primary model and Qwen2.5-7B as a secondary one. All other analyses are secondary: the premature-conclusion rate, added after the confirmatory results as a descriptive endpoint computed from fields logged in every episode; the other models and sampled decoding; the locator-uncertain and text-evidence splits, including the verified-attestation run; the flat gate and the after-the-fact filter; and the regulatory evaluation. No registered test depends on them, and their intervals and tests are not corrected for multiplicity.

\paragraph{Deviations.} Five deviations were logged with dates and reasons; three concern the analyses reported here: (1) confirmatory runs were queued behind the development pilot, with no confirmatory result computed before the minimum effect was frozen; (2) an emission-rate column and a flag for agents answering under 25\% of answerable tasks were added after the first 7B results, before any 32B confirmatory data; and (3) a default width limit in the route compiler was made optional after a round-trip test showed that 189 authored routes are seven atoms wide, and the test reproduces all 1,440 authored conclusions on all 74,784 realizable states. The other two concern analyses this paper omits: the registered induction-method contrasts and a descriptive failure-mode table. The registered four-agent natural-prevalence set is also not reported.

\section{Additional Results}
\label{app:results}

This appendix adds evidence for each experiment, in the order of Section~\ref{sec:experiments}. Differences reported with a confidence interval are family-mean paired differences (Appendix~\ref{app:protocol}), so they can differ from differences between the pooled values in the tables.

\subsection{The full information ladder and other models}
\label{app:res:a}

Section~\ref{sec:exp:ladder} reports selected rungs of the ladder. Of the information steps, only the ledger display changed false authorization significantly, adding 4.8 points with a 95\% confidence interval (CI) of [1.2, 8.6], and the mask removed 28.0 points (95\% CI [22.7, 33.5]). Here we show every step and test whether the enforcement result depends on the model or on greedy decoding.

\begin{figure*}[t]
\centering
\includegraphics[width=0.9\textwidth]{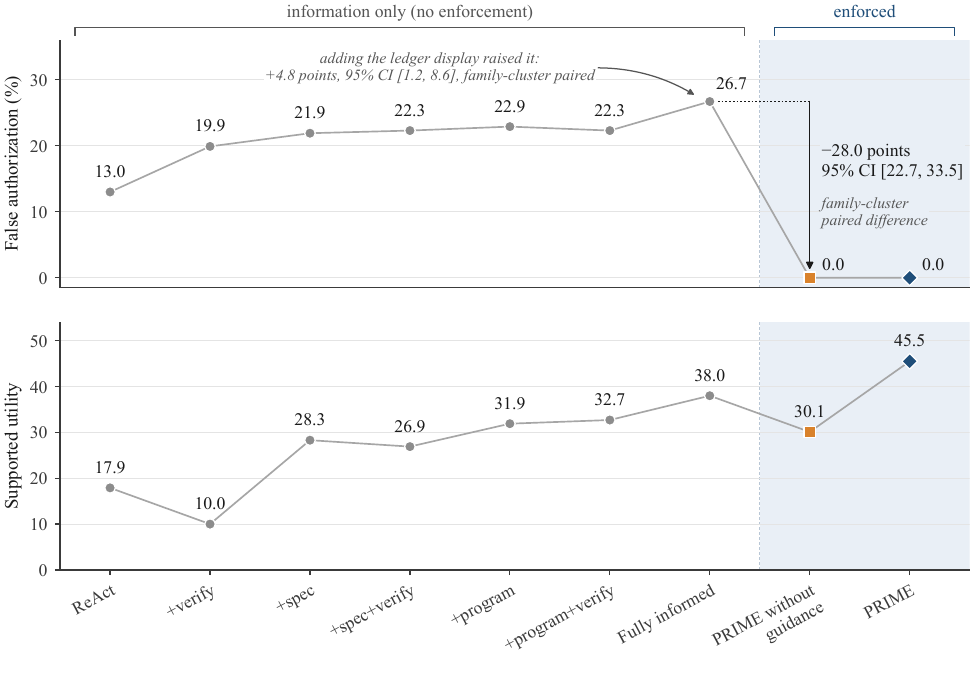}
\caption{Information ladder on 480 tasks at 20 calls. Top shows false authorization and bottom supported utility. Annotations are family-mean paired differences with 95\% intervals.}
\label{app:fig:ladder}
\end{figure*}

\textbf{Every Qwen2.5-32B-Instruct agent without the mask stated unsupported conclusions, and the mask removed them for every model and decoding setting tested.} Figure~\ref{app:fig:ladder} shows all nine agents, including the +spec and +program steps without the verification instruction. From +verify to the fully informed agent, false authorization stayed between 19.9\% and 26.7\% while utility rose from 10.0 to 38.0, and the mask is the only step that brings false authorization to zero. Information raised both quantities because the agent found more fields and stated a conclusion in more episodes: Table~\ref{tab:main} shows emission rising from 59.0\% for ReAct to 99.0\% for the fully informed agent, and each additional statement is another occasion to overclaim.

Table~\ref{app:tab:transfer} in Section~\ref{sec:exp:ladder} gives the other models. For gpt-oss-20b and Mistral-Small-24B, the mask cost family-mean utility of 2.5 (95\% CI [0.9, 4.2]) and 5.6 ([3.1, 8.1]) relative to the fully informed agent. On the security domain alone, it removed 31.8 points of false authorization for Qwen2.5-32B. The sampled-decoding runs of Section~\ref{sec:exp:ladder} use temperature 0.7 and three seeds. These settings change how the agent behaves but not the check the mask applies, which is why the enforced rate is zero in every run while the unenforced rate varies with how often the agent states a conclusion.

\subsection{Registered guidance cells and alternatives to guidance}
\label{app:res:b}

Section~\ref{sec:exp:guidance} reports the guidance gain over all tasks at 20 calls. Under the mask alone, utility fell from the fully informed agent's 38.0 to 30.1 (family-mean cost 7.4, 95\% CI [4.5, 10.3]), and guidance raised it to 45.5 (gain 15.6, 95\% CI [12.5, 18.6]). This comparison adds the gate and also puts the list of authorized conclusions in the prompt, so its utility change reflects both. At 20 calls, the runtime's final statement supplied 35.0 of PRIME's 45.5 points, and completion added 0.5 to 0.9 calls per episode. The registered analysis tested the gain separately at each of the four budgets on nested and branching tasks, against the minimum effect of 6.16 points that the registered protocol of Appendix~\ref{app:protocol} fixed before scoring.

\begin{figure*}[t]
\centering
\includegraphics[width=0.85\textwidth]{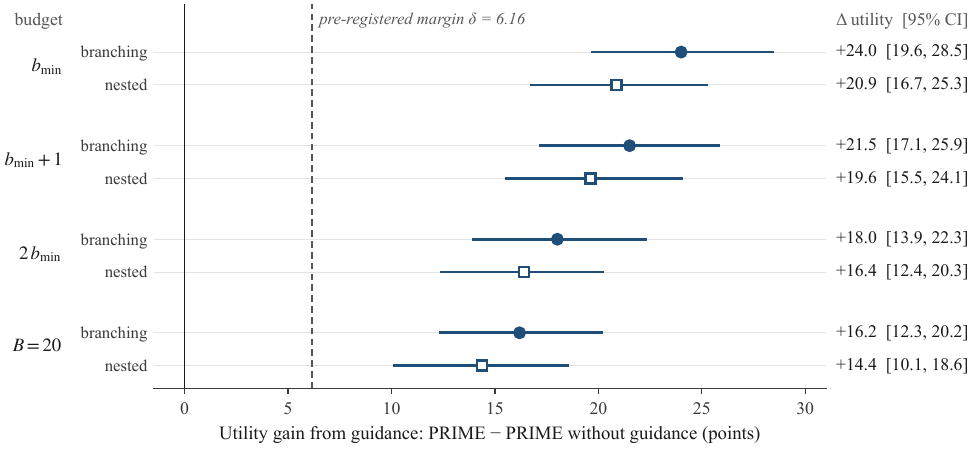}
\caption{Guidance gains by budget and route structure (480 tasks). Points show family-mean paired utility differences with 95\% intervals. Dashed line marks the minimum effect of 6.16.}
\label{app:fig:controllaw}
\end{figure*}

\textbf{Guidance raised utility by more than the minimum effect in all eight registered Qwen2.5-32B-Instruct cells.} Figure~\ref{app:fig:controllaw} shows that the family-mean gain of PRIME over PRIME without guidance fell with the budget, from 24.0 [19.6, 28.5] at $b_{\min}$ to 16.2 [12.3, 20.2] at 20 calls on branching tasks and from 20.9 [16.7, 25.3] to 14.4 [10.1, 18.6] on nested ones; even the lowest interval bound, 10.1, lies above 6.16. For Qwen2.5-7B, the registered secondary model, the eight gains were 0.6 to 3.1 points, every 95\% interval ended below 6.16, and every interval on nested tasks included zero. The gain falls with the budget because the guided agent stops sooner when it may call more, whereas the agent without guidance stops at its first permitted conclusion at every budget. Branching tasks, where a failed check can leave a second route to a stronger conclusion, gained more at every budget. The scripted references depend on the budget in the same way: the random-order policy scored 46.5 at $b_{\min}$ and matched the ceiling-targeting policy's 59.2 from $2b_{\min}$, so the order of checks matters only when the budget is tight.

\paragraph{An ideal after-the-fact filter.} A common alternative to enforcement during the episode is to let the agent answer freely, reject an unsupported answer afterward, and let the agent try again. With a perfect checker and unlimited retries, such a filter can at best state the strongest conclusion that the evidence the agent gathered supports, because retrying changes what the agent says, not what it checked. We computed this bound from the fully informed agent's logged episodes.

\begin{figure*}[t]
\centering
\includegraphics[width=0.8\textwidth]{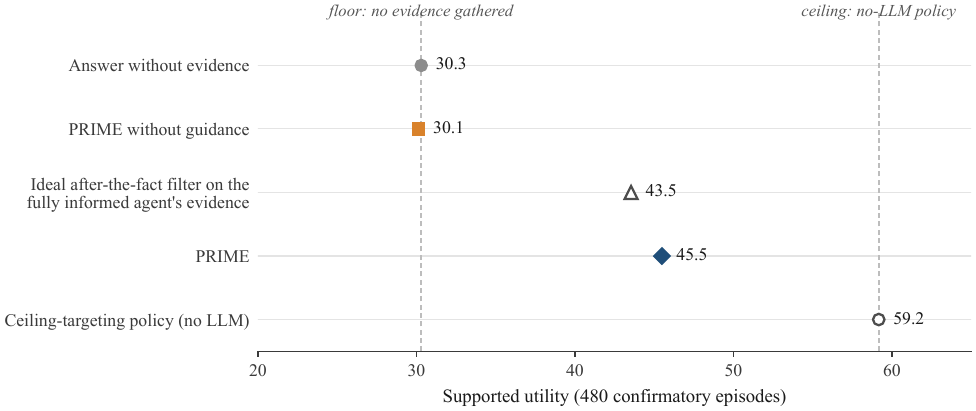}
\caption{Utility on 480 tasks at 20 calls, comparing direct answering, unguided PRIME, the ideal filter on the fully informed agent's evidence, PRIME, and scripted ceiling targeting.}
\label{app:fig:rejection}
\end{figure*}

\textbf{Even this ideal filter fell short of PRIME, because a filter cannot change which evidence is gathered.} Figure~\ref{app:fig:rejection} shows that the filter reached 43.5, against 45.5 for PRIME and 59.2 for the ceiling-targeting policy; answering without evidence scores 30.3, and PRIME without guidance 30.1. The mask alone falls below the ideal filter because refusing a conclusion leaves the agent with its first permitted answer, whereas the filter credits the strongest conclusion the agent's evidence supports. Guidance reverses the comparison because it acts while checks can still be made.

\paragraph{A gate on a single conclusion.} PRIME orders conclusions by strength, so partial evidence can still earn a weaker conclusion. To test whether this order matters, the \emph{flat gate} keeps the prompt of PRIME without guidance but exposes only the strongest conclusion and admits it once a route to it is established; otherwise the agent abstains. It has no graded conclusions, no guidance, and no completion.

\textbf{Without graded conclusions, the gate stated nothing unsupported but lost most of the useful output.} On the 480 balanced tasks at 20 calls, the flat gate reached utility 8.1 and stated a conclusion on 9.8\% of answerable tasks, against 30.1 for PRIME without guidance and 45.5 for PRIME (PRIME's gain 37.3, 95\% CI [32.6, 42.0], 60 of 60 families). The runtime refused 364 emission attempts, each either the conclusion before its route was established or a weaker one, and the agent stated the conclusion on 39 tasks. On the 152 tasks whose strongest conclusion is supported, the gate reached 25.7, against 38.8 for PRIME without guidance and 65.8 for PRIME; on the other 328 it can score nothing by construction, where the graded agents score 26.1 and 36.1. The order lets partial evidence earn a weaker conclusion, and guidance carries the agent from there to a stronger one.

\subsection{Field choice and other models}
\label{app:res:c}

The evidence-location results of Section~\ref{sec:exp:locator} measured the benefit of the agent's choice of field with one model, and the guidance results showed that the agent's stopping limits what guidance can add. Both behaviors belong to the model, so we repeated the experiments of Sections~\ref{sec:exp:ladder} and~\ref{sec:exp:guidance} with gpt-oss-20b and Mistral-Small-24B in all four domains.

\begin{table*}[t]
\centering
\footnotesize
\setlength{\tabcolsep}{1.8pt}
\begin{tabular}{llrrrrrrr}
\toprule
 & & FA (\%) & \multicolumn{3}{c}{Utility, balanced split} & Guidance gain & \multicolumn{2}{c}{Locator split, $b_{\min}$} \\
\cmidrule(lr){4-6}\cmidrule(lr){8-9}
Model & Agent & at 20 & $b_{\min}$ & $2b_{\min}$ & 20 & at $b_{\min}$ [95\% CI] & Utility & vs.\ blind [95\% CI] \\
\midrule
Qwen2.5-32B & Fully informed agent & 26.7 & 38.1 & 38.6 & 38.0 &  & 37.2 & $-$2.0 [$-$4.7, 0.8] \\
 & PRIME without guidance & 0.0 & 30.1 & 30.1 & 30.1 &  & 22.0 & $-$16.5 [$-$19.9, $-$13.1] \\
 & PRIME & 0.0 & 52.6 & 47.3 & 45.5 & 23.0 [19.5, 26.3] & 51.5 & 12.5 [8.9, 16.3] \\
 & PRIME with completion & 0.0 & 58.8 & 59.2 & 59.2 &  & 53.7 & 14.7 [11.5, 18.0] \\
\midrule
gpt-oss-20b & Fully informed agent & 17.7 & 30.6 & 30.5 & 30.5 &  & 31.5 & $-$7.5 [$-$9.7, $-$5.2] \\
 & PRIME without guidance & 0.0 & 26.9 & 26.9 & 26.9 &  & 25.7 & $-$12.8 [$-$15.3, $-$10.2] \\
 & PRIME & 0.0 & 55.7 & 58.8 & 58.5 & 28.8 [25.7, 31.9] & 52.6 & 14.2 [11.1, 17.7] \\
 & PRIME with completion & 0.0 & 56.1 & 58.6 & 58.6 &  & 52.2 & 13.8 [10.7, 17.2] \\
\midrule
Mistral-Small-24B & Fully informed agent & 19.4 & 34.4 & 34.4 & 35.2 &  & 32.3 & $-$6.4 [$-$9.1, $-$3.8] \\
 & PRIME without guidance & 0.0 & 29.3 & 29.2 & 29.3 &  & 25.4 & $-$13.4 [$-$16.0, $-$10.9] \\
 & PRIME & 0.0 & 37.9 & 36.7 & 36.6 & 10.3 [7.3, 13.3] & 33.6 & $-$4.7 [$-$7.8, $-$1.6] \\
 & PRIME with completion & 0.0 & 55.8 & 59.1 & 59.1 &  & 42.6 & 2.8 [0.7, 4.9] \\
\bottomrule
\end{tabular}
\caption{Three-model results on balanced and locator-uncertain splits. FA is measured at 20 calls and utility at three budgets. Differences are family-mean contrasts, with blind enumeration as the locator baseline (40.3). Locator completion is live, unlike the Qwen replay (53.6). The guidance gain is defined only for PRIME.}
\label{app:tab:multimodel}
\end{table*}

\textbf{Guidance helped every model, but only some models chose fields better than fixed-order checking.} Table~\ref{app:tab:multimodel} shows that guidance raised utility at $b_{\min}$ over the mask alone by 28.8 points (95\% CI [25.7, 31.9]) for gpt-oss-20b and by 10.3 ([7.3, 13.3]) for Mistral-Small-24B. On the locator-uncertain copy, gpt-oss-20b's choice among the candidate fields beat blind enumeration by 14.2 ([11.1, 17.7]) at $b_{\min}$, as Qwen2.5-32B-Instruct's did, but Mistral-Small-24B's fell 4.7 ([$-$7.8, $-$1.6]) short of it. For Qwen2.5-32B-Instruct, PRIME's advantage over blind enumeration at $b_{\min}$ (family-mean difference 12.5, 95\% CI [8.9, 16.3]) held in 56 of the 60 families; the mask alone, with the same candidates but no guidance, reached 22.0, and PRIME's gain over it was 29.0 (95\% CI [25.5, 32.6]).

Two behaviors of the model explain these differences. On the balanced split, Mistral-Small-24B's agent stopped after fewer calls (2.9 at $b_{\min}$, against 4.2 for gpt-oss-20b and 3.4 for Qwen2.5-32B-Instruct) and stated no conclusion in most episodes, so the runtime's statement at the stopping point carried most of its utility and guidance had less to add. On the locator-uncertain copy, it also chose fields less accurately. Its first read hit the correct field for 44.5\% of atoms, against 71.2\% for Qwen2.5-32B-Instruct and 81.0\% for gpt-oss-20b (33\% for a first try in listed order), and 29.2\% of its reads went to wrong candidates, against 14.7\% and 10.6\%. Completion compensates for both early stopping and poor field choices. On the balanced split, it matched the ceiling-targeting value at $2b_{\min}$ and 20 calls for Mistral-Small-24B and came within 0.6 points of it for gpt-oss-20b. At $b_{\min}$ it stayed below by family-mean differences of 2.4 points for gpt-oss-20b and 4.0 for Mistral-Small-24B. On the locator-uncertain copy, completion tries the candidates the agent has not read, and with it Mistral-Small-24B finished 2.8 points ([0.7, 4.9]) above blind enumeration. The benefit of letting the model choose where to look is therefore a property of the model, and completion limits the cost of a poor choice. A live run of completion with Qwen2.5-32B-Instruct at $b_{\min}$ scored 53.7, within 0.1 of the replayed 53.6, and Table~\ref{app:tab:multimodel} reports live runs for all three models.

\subsection{Text judgments and verification}
\label{app:res:d}

Each task of the text-evidence split has one field atom (that the title or summary is present) and three text atoms. For clinical registrations, the source is the eligibility-criteria text, and the text atoms are that enrollment is restricted to adults (minimum age 18 years or more), that it is restricted to one sex, and that healthy volunteers are accepted; $c_2$ requires the first, and $c_3$ the first true and the other two false. For security advisories, the source is the advisory's details text, and the text atoms are that it names the affected package (by its registered name or a distinctive part of it, not only the vendor or ecosystem), that it states a version in which the issue is fixed, and that it states the first-affected version; $c_2$ requires the first, and $c_3$ all three. Each atom's truth is computed from a structured field that the embedded record omits (minimum age, sex, and healthy-volunteer flag; the registered package name and the fixed and first-affected versions of the advisory's range data, the name counting as present when it or one of its distinctive tokens occurs in the text), so a task violates an atom whenever the registry does; records without the structured field (no fixed version, every version affected) are not used. The agent reads the text with \texttt{read\_text} and attests with \texttt{attest\_atom}. The runtime admits an attestation only for a text-evidence atom when the text was read in the episode and the quote is a verbatim span of at least eight characters. Otherwise, it rejects the call and leaves the atom open.

\paragraph{Rejected and wrong attestations.} The runtime rejected 15 of 225 attestation calls for the fully informed agent, 65 of 216 for PRIME without guidance, and 72 of 370 for PRIME. All 72 PRIME rejections involved nonverbatim quotations. Of the other 80 rejections, 14 involved nonverbatim quotations and 66 attempted attestations on non-text atoms. No call was rejected for an unread source. Rejected calls never enter the ledger. Among admitted attestations, wrong values ran in both directions. Under PRIME, adults-only enrollment was judged false for 5 registrations whose minimum age is 18 or more and true for 6 whose minimum age is lower. The one-sex and healthy-volunteer conditions were judged false for 5 whose structured fields hold them. In the advisories, the package, fixed version, or first-affected version was judged named for 11 whose text does not contain it and absent for 3 whose text does. Each unsupported conclusion PRIME stated was authorized by a route containing a wrongly attested atom, while two of the fully informed agent's unsupported conclusions were statements its ledger did not support. Because each domain is one family, contrasts on this split are at the record level. Against the mask alone, PRIME changed utility by $+10.1$ [7.8, 12.6] and false authorization by $+2.9$ [0.8, 5.4] (7 of 240 tasks, exact McNemar $p=0.016$), and against the fully informed agent by $-1.0$ [$-3.5$, 1.5] and $+0.4$ [$-1.7$, 2.5].

\paragraph{Verification.} The verified condition tests whether screening attestations before admission removes these errors, and at what cost. Before an attestation that passes the runtime's checks enters the ledger, the same model is asked once more, in a fresh context that contains only the atom's condition and the source text (not the agent's value, quote, or conversation), whether the text establishes the condition. The attestation is admitted only if the answer equals the agent's value; otherwise the atom stays open, and the agent is told only that an independent check did not confirm it. The prompt and rule were frozen before the run. The arm, budget, model, and serving stack are those of the PRIME run in Table~\ref{tab:text}, and the comparison is with that run.

\textbf{Screening removed most wrong judgments before they reached the ledger, at the cost of refusing many correct ones.} The last row of Table~\ref{tab:text} shows that false authorization fell from 4.6\% to 0.8\% (2 of 240 tasks, 95\% CI [0.2, 3.0]) and that 6 of the 222 admitted attestations were wrong (2.7\%, against 30 of 298), while supported utility fell from 57.5 to 55.1 (paired loss 2.4 [0.3, 4.3]). Emission was 100\% in both runs. The verifier refused 65 attestations in 59 episodes, 46 with values that match the benchmark reference and 19 with wrong values. It refused all 12 wrong attestations in clinical registrations but 7 of the 13 in security advisories, and it refused right attestations of all three text atoms. After a refusal, every episode still ended with the runtime stating the strongest authorized conclusion, never with abstention. The final conclusion was lower than in the PRIME run of Table~\ref{tab:text} on 43 tasks and higher on 6 (32 and 0 of them in the 59 episodes with a refusal), and the refusals prevented 8 of the 11 unsupported conclusions of that run. A ninth did not recur either, although its episode had no refusal. Verification added 1.2 model calls and 3.6\% more tokens per episode on average. These calls fall outside the 20-call evidence-tool budget, which is therefore not a matched total compute budget.

The two remaining unsupported conclusions, the same in both runs, involve product mentions in advisory text (``AWS API MCP Server''; ``Zope \ldots{} Plone'') that differ from the structured registry identifiers. The text-evidence score therefore measures agreement with the benchmark's structured reference for short, specified conditions read from individual fields; the released logs keep the passages for case-level inspection. Code, traces, and the frozen prompt are in \texttt{experiments/blind\_verification/} of the released code.

\subsection{The closest systems}
\label{app:extsys}

Section~\ref{sec:exp:extsys} runs five of the systems discussed in Appendix~\ref{app:related} on the 480 balanced tasks. Every system receives the same task prompt, the same requirement content (the committed program's conclusions, their order, and each conclusion's routes over checks, with each check's tool and field), and the same evidence: check values come from the same deterministic resolvers applied to the same tool observations. Each system receives the requirements in its own specification format, compiled mechanically from the program, and keeps its own control: what decides the next check, when to stop, and which conclusion is reported. All use Qwen2.5-32B-Instruct with greedy decoding and the 20-call budget, counted as for PRIME; their own extra model calls (planning turns, re-plans, planner decisions, judge calls) are outside the budget and reported in Table~\ref{app:tab:extsys}. For the fully informed agent and PRIME, the model calls in that table are agent turns, reminders after a silent stop, and the forced final question, and completion adds none. No system has a human in the loop, so none is offered the ask-the-user tool.

\paragraph{Instantiations.} \emph{AgentSpec} \citep{wang2026agentspec} adds no text to the agent's prompt, so the protected agent is the fully informed agent of Section~\ref{sec:exp:ladder}, the agent to which PRIME's mask is added. One rule guards its conclusion action: \texttt{rule @no\_unsupported\_conclusion trigger emit\_claim check !conclusion\_established enforce stop end}, whose predicate holds when a route of the claimed conclusion has all of its checks recorded as true. We run the enforcement the paper evaluates, \emph{stop}, which ends the episode without a conclusion, and the corrective enforcement it describes, \emph{self-examine}, which returns the violated rule to the model in the released comment format and lets it re-plan up to four times per step. \emph{ResearchLoop} \citep{researchloop2026} follows its Algorithm~1: the request is the research question; the conclusions are registered in a claim ledger; one task contract per conclusion is queued from the weakest; code evaluates the evidence gates after every call (a route gate requires all of its checks, and a conclusion's gate passes when any of its route gates does); a task that produces no evidence twice is blocked; and the report contains the admitted conclusions. The model receives the paper's \emph{full} condition prompt with the domain nouns replaced and writes an explicit plan, which the runtime keeps with the run state. \emph{Proof-or-Stop} \citep{prooforstop2026} runs the lifecycle none, $c_1$, $c_2$, $c_3$ (done): a conclusion is an advance request that a code gate admits only when the recorded checks establish one of its routes; a refused request returns a block report that lists each check as established, failed, or missing; the agent may repair twice; the controller, not the agent, ends the loop, at done, at the repair cap, at the budget, or after more than three consecutive turns without an action; and the highest admitted stage is reported. Our adaptation of \emph{HESP} \citep{liu2026hesp} keeps its configuration except the probe tables, as described below. It has hypotheses ``the strongest supported conclusion is $c_k$'' plus a catch-all hypothesis, \emph{other},, a Bayesian ledger over the check values with base rates counted on the development split, probes (field reads) chosen by expected information gain per cost, and a controller that concludes when the leading hypothesis reaches posterior 0.8 with a supporting observation. The model receives HESP's released prompt and can only abstain. \emph{GLEAN} \citep{zhang2026glean} verifies the fully informed agent's logged trajectories: a same-model judge rates every step against each route of the stated conclusion with the paper's prompt, the ratings are accumulated and calibrated by Bayesian logistic regression on 100 development trajectories, uncertain cases receive the paper's differential checks, and the conclusion is stated only when its calibrated probability of being correct is at least 0.5.

\paragraph{Code and fidelity.} As of October 2026, HESP's code is released under the MIT license. AgentSpec's and ResearchLoop's repositories carry no license; AgentSpec's rule engine runs only LangChain text agents, and ResearchLoop's repository does not contain the controller its documentation invokes. Proof-or-Stop and GLEAN have no public code. We therefore reimplemented each decision procedure on the shared tool loop, reusing released prompts and constants where they exist, and record with the code every choice the papers leave open. Deviations from the papers: greedy decoding throughout; HESP's fixed per-probe tables were replaced by exact elimination over check values, because a fixed table cannot represent a conjunction of checks (with it, the posterior of $c_3$ stays near 0.56 after every check is read); and GLEAN's guideline expansion is unavailable because the routes of the stated conclusion are its only guidelines.

\begin{table}[t]
\centering
\scriptsize
\setlength{\tabcolsep}{2.2pt}
\begin{tabular}{lrrrrrr}
\toprule
System & FA\,$\downarrow$ & Prem.\,$\downarrow$ & Util.\,$\uparrow$ & Emit\,$\uparrow$ & Reads\,$\downarrow$ & Calls\,$\downarrow$ \\
\midrule
Fully informed agent & 26.7 & 28.3 & 38.0 & 99.0 & 3.09 & 4.4 \\
\midrule
AgentSpec, stop & 0.0 & 0.0 & 38.7 & 87.4 & 2.96 & 4.2 \\
AgentSpec, self-examine & 0.0 & 0.0 & 40.1 & 92.9 & 3.06 & 5.8 \\
ResearchLoop & 0.0 & 0.0 & 47.7 & 97.0 & 3.22 & 6.0 \\
Proof-or-Stop & 0.0 & 0.0 & 52.0 & 99.2 & 4.21 & 12.3 \\
GLEAN & 14.4 & 14.6 & 35.0 & 83.3 & 3.09 & 9.2 \\
HESP & 0.0 & 11.7 & 59.2 & 100.0 & 3.91 & 3.9 \\
\midrule
PRIME & 0.0 & 0.0 & 45.5 & 100.0 & 2.75 & 5.5 \\
PRIME with completion & 0.0 & 0.0 & 59.2 & 100.0 & 3.61 & 5.5 \\
\bottomrule
\end{tabular}
\caption{Closest-system results on 480 tasks at 20 calls. FA and Prem.\ are episode percentages, Emit is among answerable tasks, and Reads counts fields. Calls counts model turns, planning, and reminders as applicable, and for GLEAN only its judge calls on the fully informed agent's runs.}
\label{app:tab:extsys}
\end{table}

\textbf{Gates on the routes prevented unsupported conclusions but stopped short of the strongest supported one.} Table~\ref{app:tab:extsys} shows that AgentSpec, ResearchLoop, and Proof-or-Stop stated nothing unsupported and nothing their reads had not established. AgentSpec's stop rule fired in 21.7\% of episodes and ended them without a conclusion. Self-examination recovered some of these, at 5.8 model calls per task, but both stayed below PRIME without completion (by family-mean differences of 7.5 and 6.1 points). ResearchLoop's task queue led the agent from the weakest conclusion upward, and its final report contained only admitted conclusions. In 14.4\% of episodes, the agent's own report claimed more than its gates admitted. ResearchLoop reached 47.7, a family-mean 3.0 points above PRIME without completion (95\% CI [$-$0.3, 6.3]). Proof-or-Stop drove the agent toward $c_3$ and read more fields than PRIME's agent alone (4.2 against 2.8 per task, at 12.3 model calls), which put its utility a family-mean 7.7 points above PRIME without completion (95\% CI [5.0, 10.5]). Under PRIME's own cap, which answers at most three turns without an action per episode, its utility changed by 0.3 points, so the difference does not come from its more lenient cap. None of the three runs the remaining checks itself. Running them lifts PRIME with completion to 59.2, family-mean differences of 10.3 points over ResearchLoop and 5.6 over Proof-or-Stop (95\% CIs [7.7, 13.2] and [3.8, 7.5]).

\begin{table}[t]
\centering
\scriptsize
\setlength{\tabcolsep}{3pt}
\begin{tabular}{lrrrr}
\toprule
 & $b_{\min}$ & $b_{\min}{+}1$ & $2b_{\min}$ & 20 \\
\midrule
\multicolumn{5}{l}{\emph{Balanced split: utility $\uparrow$}} \\
\quad Ceiling-targeting (no model) & 59.2 & 59.2 & 59.2 & 59.2 \\
\quad HESP controller & 58.0 & 59.2 & 59.2 & 59.2 \\
\multicolumn{5}{l}{\emph{Locator-uncertain copy: utility $\uparrow$}} \\
\quad PRIME & 51.5 & 50.7 & 48.8 & 48.2 \\
\quad PRIME with completion & 53.7 & 54.6$^{r}$ & 59.2$^{r}$ & 59.2$^{r}$ \\
\quad Blind enumeration (no model) & 40.3 & 44.7 & 58.5 & 59.2 \\
\quad HESP controller & 11.2 & 12.3 & 50.3 & 59.0 \\
\quad True locators (no model) & 59.2 & 59.2 & 59.2 & 59.2 \\
\bottomrule
\end{tabular}
\caption{Budgeted PRIME--HESP comparison (480 tasks). The HESP and no-model controllers use scripted checks. The marker $^{r}$ denotes completion replay from PRIME stopping points. The balanced-split premature-conclusion rate is 0.0\% for ceiling targeting and 11.7\% for HESP at every budget.}
\label{app:tab:hesp}
\end{table}

\textbf{HESP matched PRIME with completion where every check was code reading a named field, and fell behind where budgets were tight or fields had to be located.} HESP's controller chose every probe and issued all 480 verdicts, and the model never abstained, so the live runs equal the controller's. Its utility equaled that of PRIME with completion on every task at 20 calls: both stated the strongest supported conclusion everywhere, which gives the ceiling of 59.2 that the scripted policies of Section~\ref{sec:exp:guidance} also reach at this budget, since every task has at most eight atoms; but in 11.7\% of tasks (95\% CI [6.4, 15.3]) it stated a conclusion its reads had not established, and it read a family-mean 0.35 more fields per task (95\% CI [0.23, 0.46]). When only conclusions that a system's own reads established are credited, PRIME with completion scored 59.2 and HESP 49.0 (family-mean difference 8.7, 95\% CI [5.0, 12.7], better on 56 tasks and worse on none). HESP's conclusions rest on a posterior: with base rates counted on development data, or on the other three domains only (utility 59.2, 10.2\% premature), the premature conclusions were correct, but with a uniform prior within each hypothesis it stated an unsupported conclusion on 5.0\% of tasks (utility 49.4), and with independent fair-coin checks it never reached the threshold. Raising HESP's threshold removes the premature conclusions at a cost: at 0.9 none were premature and utility stayed at 59.2 at 20 calls, but it read 4.6 fields per task (PRIME with completion, 3.6) and scored 50.4 at $b_{\min}$; at 0.95 it scored 36.6 at $b_{\min}$; at 0.99 it never concluded. PRIME reports what the routes establish, so it needs no such threshold. Table~\ref{app:tab:hesp} compares the two controllers by budget, with PRIME at $b_{\min}{+}1$ run for that table. At $b_{\min}$, the ceiling-targeting policy scored 59.2 against 58.0 for HESP's controller (family-mean difference 1.6, 95\% CI [0.6, 2.7], with 17 tasks better and none worse) with fewer reads (3.4 against 3.9), and the live PRIME with completion scored 58.8 (family-mean difference 1.1, $p=0.04$). On the locator-uncertain copy, HESP's controller cannot read the descriptions and abstains unless one hypothesis is probable enough: it scored 11.2 at $b_{\min}$, against 53.7 for PRIME with completion (family-mean difference 41.8, 95\% CI [37.1, 46.5]) and 40.3 for blind enumeration. At 20 calls, blind enumeration reached 59.2, while HESP reached 59.0 and stated one unsupported conclusion.

\textbf{GLEAN reduced unsupported conclusions but could not remove them.} Calibrated on 100 development trajectories, where its area under the receiver operating characteristic curve (AUROC) was 0.85, the verifier ranked the held-out stated conclusions with AUROC 0.84 and expected calibration error 0.06, and its differential checks fired on 49\% of them. It accepted 380 of the 466 stated conclusions, 69 of the 128 unsupported and 311 of the 338 supported, so false authorization fell from 26.7\% to 14.4\% and utility from 38.0 to 35.0, at 9.2 judge calls per task. A verifier of a finished trajectory can accept or reject the agent's conclusion, but it cannot direct the agent to the checks that would establish a stronger one.

\section{Externally Derived Regulatory Requirements}
\label{app:fdaaa}

\paragraph{Summary.} Section~\ref{sec:exp:regulation} gives the question, the setup, and the main results. This appendix describes the regulation and records, the two construction workflows, how each program was validated and frozen, the held-out sample, and the scoring, and then explains the results. The tracker is not a legal adjudication. The assistant-built program's reading of the deadline rules was chosen by agreement with the tracker on a development pool. A stated $c_k$ is \emph{unsupported} if $k$ exceeds the tracker's level (3: reported late; 2: due; 1: otherwise), and \emph{coverage} is the share of registrations with a stated conclusion. Under the scripted ceiling-targeting policy on all 1,999 registrations, the protocolized program covered 92.4\% with 3.5\% unsupported (95\% CI [2.8, 4.4]), and the assistant-built one 26.3\% with 1.8\%.

\paragraph{Regulation and records.} The source is the U.S. clinical-trial results-reporting regulation, 42 CFR Part~11 \citep{cfr2016part11}, which implements the results-reporting provisions of the Food and Drug Administration Amendments Act of 2007 (FDAAA). Part~11 was amended in December 2024 to update referenced web addresses \citep{cfr2024part11amend}. The programs encode its sections 11.22(b), 11.42, and 11.44(a) to (c). The records are ClinicalTrials.gov registrations, read through the registry's public API (version 2). The development pool is a snapshot of 6,000 registrations frozen on 3 August 2026, with its inputs pinned by hash. The reference is the FDAAA TrialsTracker \citep{devito2018trialstracker,devito2020compliance}, the reporting-status tracker of the summary, which classifies each tracked trial from the same registry records as reported, reported late, overdue, or ongoing. We use its export of 3 August 2026, which lists 49,755 trials. Because the tracker reads the same registry, the comparison tests the rule and the decision procedure, not the data.

\paragraph{Two construction workflows.} In the \emph{assistant-built} workflow, an LLM coding assistant (Anthropic Claude) wrote the program as code: the three conclusions, six checks that read the study type, start date, regulatory status of the product, primary completion date, results-submission date, and certification date, and the routes $\{e_1,e_2,e_3\}$, $\{e_1,\ldots,e_4\}$, and $\{e_1,\ldots,e_5\}$. No field path, predicate, or route was written or edited by a person. Every check runs through the runtime's unextended ten-predicate grammar, and each registration's reference values are computed by the runtime's own evaluator. In the \emph{protocolized} workflow, protocol and code were committed before any run. Qwen2.5-32B-Instruct received selected definitions of 42 CFR 11.10(a), the text of 11.22(b), 11.42, and 11.44(a) to (c), a catalog of the registry fields, and three unlabeled registrations. It wrote programs of at most 12 single-field atoms, routes for the three conclusions (here $c_2$ requires that the deadline has passed by the snapshot date), and a sandboxed function for derived values such as deadlines; each round drew one greedy and eight sampled candidates.

\paragraph{Validation and freezing.} The assistant-built program's deadline logic was refined over four cumulative readings of the statute: a tracker emulation, then the calendar-year deadline of 11.44(a), then the certification-timeliness condition of 11.44(b)(1), then removal of an administrative grace period. Each reading was scored for agreement with the tracker on the 6,000 snapshot registrations (94.8\%, 94.8\%, 94.3\%, and 98.7\%), and the last, which agreed best, was retained; its 76 residual disagreements were inspected and traced to the tracker's dating of certifications. No review of the program against the regulation by a person is recorded beyond this comparison, and the program was then fixed in the instance builder. The protocolized candidates were validated on 600 development-pool registrations, re-fetched in full: each candidate's full-evidence level was scored against the tracker (17 of 18 candidates valid; the best of round 0 agreed on 42.0\%), the best was returned once with 12 validation disagreements for repair, and the selected revision (71.0\%) was frozen and hashed. The frozen program checks that the study is interventional and studies an FDA-regulated drug or device, that primary completion falls after 17 January 2017, that a deadline 365 days after primary completion has passed, and that results were first submitted after it. The 365 days approximate the one-year deadline of 11.44(a) and differ from it by a day when the year contains 29 February. The frozen program omits the phase-1 and feasibility exclusions and the U.S.-connection condition of 11.22(b), and the certification extensions of 11.44(b) and (c).

\paragraph{Test samples.} The assistant-built program was first tested on the development pool. Of the 6,000 snapshot registrations, it classified 4,990 as not subject to the requirement; of the remaining 1,010, 481 belonged to sponsors with a single trial in the pool or fell outside a cap of 600 and were not used, and 210 could not be classified because an applicability field could not be read by the deployed predicates. The remaining 319 registrations, from 124 sponsors, were written as instances, validated against the task schema, and committed before any agent ran, and the agent runs saw no tracker labels. Because the reading of the statute was chosen against the tracker on this pool, agreement on the 319 is not out of sample with respect to that choice. The held-out sample removes this dependence. With the builder, its predicates, the selection rules, and the scoring rule frozen and the protocol fixed before any label was read, we sampled 2,000 of the 43,755 trials in the tracker export that lie outside the development pool and come from sponsors absent from it (seed 20261010), fetched each from the public API on 10 October 2026, and stored it verbatim; 1,999 fetches succeeded. Both programs were run on all 1,999 registrations, and the agents on a random subset of 600 of them; the tracker's labels were read once, after every episode had run. Two earlier runs of the protocolized test were discarded before any test label was read, because the sandbox blocked modules that the date library loads; the protocol's deviation log lists the fixes. Table~\ref{app:tab:fdaaa:attrition} shows how the assistant-built program's own exclusions, all based on registry fields and applied before any label was read, reduce the held-out sample to the 526 registrations it classifies; for the 353 whose applicability field the deployed predicate grammar cannot read (for instance a month-precision start date), no attestation is made, so the program abstains, as it did for 210 of the development pool's 1,010.

\begin{table}[t]
\centering
\footnotesize
\setlength{\tabcolsep}{4pt}
\begin{tabular}{lr}
\toprule
Held-out stratum & Value \\
\midrule
Sampled; fetched & 2,000; 1,999 \\
Not an applicable clinical trial (11.22(b)) & 370 \\
Not yet due at the snapshot date & 750 \\
Applicable and due & 879 \\
\quad applicability field unreadable & 353 \\
\quad classified & 526 \\
\midrule
Classified, of applicable and due & 59.8\% \\
Classified, of fetched & 26.3\% \\
Tracker agreement (518/526) & 98.5\% \\
\quad original statuses only (499/501) & 99.6\% \\
\quad unchanged since export (495/503) & 98.4\% \\
\bottomrule
\end{tabular}
\caption{Assistant-built held-out sample. Program exclusions precede tracker labels. The first two tracker-agreement rates have 95\% intervals [97.0, 99.2] and [98.6, 99.9].}
\label{app:tab:fdaaa:attrition}
\end{table}

\paragraph{Scoring against the tracker.} Two scores are used. The assistant-built program's classifications (compliant, late, overdue) are compared with the tracker's statuses by binary agreement: a registration agrees when both call it adverse (late or overdue for the program; reported late or overdue for the tracker) or both do not, over the registrations both classify. For the end-to-end test of both programs, unsupported conclusions and coverage are as defined in the summary, utility credits the strongest supported level as in Section~\ref{sec:bench}, and a registration without an instance counts as an abstention. Every test registration is one the tracker follows, and level 1 is not an independent determination that a trial is subject to the regulation, so this test cannot penalize calling an inapplicable trial applicable.

\paragraph{Results.} Under the ceiling-targeting policy the two programs reached utility 47.5 and 18.4, and the assistant-built program's unsupported rate of 1.8\% has a 95\% CI of [1.3, 2.5]. On the 1,112 registrations the tracker counts as due, the two stated a conclusion for 87.1\% and 44.4\%. The protocolized program's full-evidence level agrees with the tracker's on 71.6\% of these due registrations, close to its validation figure of 71.0\%, and on 94.3\% of the others, so its agreement over all test registrations is higher (81.6\%): 587 of the 600 validation registrations are due, against 1,112 of the 1,999 test registrations.

Where the assistant-built program classified a registration, its classification agreed closely with the tracker. Table~\ref{tab:fdaaa} shows agreement on 317 of the 319 development registrations (99.4\%), 52.0\% of which were late or overdue, and on 518 of the 526 held-out registrations it classified (98.5\%), with no stated conclusion lacking a route. The 98.5\% agreement and the 1.8\% unsupported rate use different definitions. Agreement compares the program's adverse or non-adverse classification with the tracker's on the 526 registrations it classified, while a conclusion is unsupported when its level exceeds the tracker's, for example when the program states that results are due for a trial the tracker still lists as ongoing, as in 25 of the 36 cases. The held-out classifications were 165 compliant, 164 late, and 197 overdue. Because the records were fetched on 10 October against the tracker's export of 3 August, a field updated in between can produce a mismatch that is not a policy difference. Of the 526, 23 carry such an update (including 3 results submissions and 1 certification), all 23 agree with the tracker, and agreement on the 503 unchanged records is 495 of 503 (98.4\%). Two of the five trials the program calls late and the tracker overdue submitted after the export. Both statuses are adverse, so they count as agreement. In the cross-tabulation of program against tracker, the registrations the program calls compliant are 146 reported and 19 ongoing, those it calls late are 157 reported late, 2 reported, and 5 overdue, and those it calls overdue are 191 overdue and 6 ongoing.

\begin{table}[t]
\centering
\footnotesize
\setlength{\tabcolsep}{4pt}
\begin{tabular}{lr}
\toprule
Quantity & Value \\
\midrule
Late or overdue under the assistant-built program & 52.0\% \\
Tracker agreement, development pool (317/319) & 99.4\% \\
Tracker agreement, held out (518/526) & 98.5\% \\
Held-out sample classified (526/1,999) & 26.3\% \\
Stated conclusions without a route & 0.0\% \\
\midrule
Filed on time (153): strongest conclusion stated & $c_2$ \\
Filed late (86): $c_3$ stated & 84.9\% \\
Never filed (80): $c_3$ stated & 0.0\% \\
\bottomrule
\end{tabular}
\caption{Assistant-built program on 319 development and 526 held-out registrations.}
\label{tab:fdaaa}
\end{table}

\paragraph{What explains the differences.} The differences trace mainly to the programs. The protocolized program omits the certification extensions of 11.44(b) and (c), and 55 of its 70 unsupported conclusions involve a filed certification. It also omits the phase-1 and feasibility exclusions and the U.S.-connection condition of 11.22(b), which this test cannot penalize. The assistant-built program's builder makes instances only for registrations it classifies as applicable and due, which limits its coverage, and 25 of its 36 unsupported conclusions are registrations the tracker keeps ongoing that the program counts as due. None of its eight held-out disagreements was updated after the export; they are clause differences: six trials whose deadline fell in the 30 days before the export, which the tracker's administrative grace period keeps ongoing and the program's reading calls overdue, and two late filers that the tracker lists as reported on time because of how it dates a certification. Of the 526 registrations, 41 carry a certification, and 3 of the 8 disagreements are among them. The two development-pool disagreements are of the same kind.

The runtime also declined to call a trial late without a filing. Table~\ref{app:tab:fdaaa:confusion} separates the program's verdict, its rule's classification of the registration, from the conclusion the runtime stated, which is what it could establish from the registry fields its checks read. On the 319 development registrations, the runtime stated $c_3$ for 73 of the 86 late filers and for none of the 80 trials that never filed. A trial that never filed has no submission event to establish $c_3$, whereas the tracker counts registry absence as non-submission, an inference an organization would have to add to the program as an explicit atom. For the other 13 late filers, the submission date appears only in an estimated field that the program does not read, so the runtime stated $c_2$. The 8 trials with an FDA Notice of Noncompliance were checked separately against the live registry by a standalone implementation of the same deadline logic; none is among the 319.

\begin{table}[t]
\centering
\footnotesize
\setlength{\tabcolsep}{2.5pt}
\begin{tabular}{lrrrcrr}
\toprule
 & \multicolumn{3}{c}{Tracker status} & & \multicolumn{2}{c}{Stated} \\
\cmidrule(lr){2-4}\cmidrule(lr){6-7}
Verdict & rep. & late & overdue & & $c_2$ & $c_3$ \\
\midrule
compliant (153) & 153 & 0 & 0 & & 153 & 0 \\
late (86) & 2 & 84 & 0 & & 13 & 73 \\
overdue (80) & 0 & 0 & 80 & & 80 & 0 \\
\bottomrule
\end{tabular}
\caption{Development registrations (319). Program verdicts versus tracker status (left) and runtime-reported conclusions (right).}
\label{app:tab:fdaaa:confusion}
\end{table}

The agents' results reflect stopping, not the rule. On the 319 development registrations, with the assistant-built program as the standard and Qwen2.5-32B-Instruct at 20 calls, the fully informed agent stated no unsupported conclusion at utility 49.1, stating $c_2$ for 151 registrations and $c_1$ for 168, never $c_3$, and stopping after four calls in 313 episodes. PRIME without completion also stated none but reached only 36.5: in 291 of the 319 episodes the agent made the three calls of $c_1$'s route and then answered in text through three reminders that named the one remaining field of $c_2$'s route, after which the runtime stated $c_1$. Every registration supports $c_2$, so stopping at $c_1$ loses 33.3 of the 66.7 points that $c_2$ earns. With completion, the runtime made the remaining calls in 299 of the 319 episodes (1.85 per episode), and PRIME reached 74.3, the scripted policies' value cell for cell. Table~\ref{app:tab:specbuild} shows the same pattern for both programs on the held-out subset of 600: PRIME with completion matched the ceiling-targeting policy (3.2\% and 1.7\% unsupported, utility 47.7 and 19.0), while PRIME alone (0.3\% and 0.2\%, utility 32.3 and 9.9) and the fully informed agent (0.0\% and 0.2\%, utility 31.2 and 13.1) stopped at weaker conclusions. Stopping at $c_1$, which this scoring credits for every test registration, also avoided most of the programs' errors. The runtime states what the committed program authorizes and no more, so agreement with the tracker depends on how faithfully the program encodes the regulation. The builder and the reading-by-reading comparison are in \texttt{src/suffice/construction/fdaaa/} of the released code, the pinned inputs in \texttt{data/fdaaa/external/}, and the held-out sample identifiers and record hashes in \texttt{experiments/fdaaa\_heldout/}.

\begin{table*}[t]
\centering
\small
\begin{tabular}{lrrrrrrr}
\toprule
 & & \multicolumn{2}{c}{ReAct} & \multicolumn{2}{c}{Static checklist} & \multicolumn{2}{c}{PRIME without guidance} \\
\cmidrule(lr){3-4}\cmidrule(lr){5-6}\cmidrule(lr){7-8}
Domain & $n$ & FA & Utility & FA & Utility & FA & Utility \\
\midrule
Software review & 21 & 47.6 & 24.6 & 0.0 & 20.6 & 0.0 & 37.3 \\
Research review & 8 & 100.0 & 4.2 & 0.0 & 10.4 & 0.0 & 37.5 \\
Security & 8 & 75.0 & 25.0 & 0.0 & 100.0 & 0.0 & 35.4 \\
Financial & 8 & 25.0 & 75.0 & 0.0 & 100.0 & 0.0 & 83.3 \\
Drug safety & 7 & 85.7 & 14.3 & 0.0 & 81.0 & 0.0 & 45.2 \\
Clinical trials & 8 & 87.5 & 18.8 & 0.0 & 39.6 & 0.0 & 45.8 \\
\bottomrule
\end{tabular}
\caption{Hand-built exploratory pilot with gpt-4.1 at 20 calls.}
\label{app:tab:sixdomain}
\end{table*}

\section{Exploratory Pilot on Hand-Built Tasks}
\label{app:pilot}

Before building SUFFICE, we ran an exploratory pilot on hand-built tasks to see whether enforcement behaved as intended outside generated records. The pilot is reported separately because its tasks, models, and agents differ from the benchmark's, and its results are never pooled with the confirmatory statistics. It has 21 software-review tasks on pinned repository commits, 8 research-review tasks, and 31 record tasks over 27 records, and every conclusion has one route. Runs used gpt-4.1 and gpt-4o-mini at 20 tool calls, except where noted. The \emph{static checklist} is an enforced agent that demands every atom on any route of a conclusion before stating it.

\textbf{Enforcement gave zero false authorization in every pilot domain.} Table~\ref{app:tab:sixdomain} shows that, with gpt-4.1, ReAct stated unsupported conclusions on 25.0\% to 100.0\% of the tasks in each domain, while the static checklist and PRIME without guidance stated none; with gpt-4o-mini on the software tasks, PRIME reached 0.0\% (utility 38.9) against ReAct's 71.4\% (34.1). Because every conclusion has one route, the static checklist and PRIME's program coincide here, so their utility differences come from the agent's behavior. The checklist beat PRIME without guidance on utility in security, financial, and drug safety; an exploratory guidance arm, added after that result was seen, raised PRIME to 100.0 in security and financial and to 76.2 against 81.0 in drug safety. The pilot's McNemar tests with Qwen2.5-7B and Qwen2.5-32B compared false authorization between PRIME with and without guidance, which is zero for both by construction, so they could not detect a guidance effect; the confirmatory study therefore tests guidance on utility.

\section{Extended Related Work}
\label{app:related}

LLM outputs, including those of agents that interleave reasoning with tool calls \citep{yao2023react}, have been made more reliable from the answer backward, through self-critique and revision \citep{shinn2023reflexion,dhuliawala2023cove}, factuality checks against retrieved evidence \citep{chern2023factool,wei2024safe}, and selective prediction or calibrated withholding \citep{geifman2017selective,mohri2024conformal}. Minimal evidence group identification determines small evidence sets that verify a claim, including alternative sets for one claim \citep{li2024meg}. GLEAN \citep{zhang2026glean} grounds verification in clinical guidelines and actively collects evidence when calibrated uncertainty is high; HALT \citep{roh2026halt} stops retrieval when evidence coverage is sufficient. These methods motivate evidence-aware checking, but PRIME studies a distinct target: the explicit logical requirements for \emph{several ordered conclusions}, including alternative sufficient routes that remain open after a failed check. Behavioral policy failures documented by \citet{schmied2025greedy}, \citet{reddy2026reasonless}, and agent benchmarks \citep{yao2024taubench} further motivate an enforcement boundary outside the model.

External controllers provide the closest foundations. Runtime shielding blocks unsafe actions \citep{alshiekh2018shielding}; AgentSpec \citep{wang2026agentspec} enforces configurable rules over agent actions and completion; Progent \citep{shi2025progent} controls tool privileges and policy changes; and \citet{doshi2026safe} derive tool-use constraints from hazard analysis. AGATE \citep{zhang2026agate} emphasizes provenance-aware protection of agent execution. These systems primarily control \emph{actions or transitions}, whereas PRIME controls \emph{which evidence-dependent conclusion may be released}. This boundary is not exclusive: AgentSpec can check agent completion, and action guards could complement PRIME's evidence resolvers.

Managing the information supplied to an agent can expose requirements, tool observations, and pending checks, but Section~\ref{sec:exp:ladder} shows that availability in the model's context is distinct from authorization of its output. PRIME keeps the ledger as runtime state and uses the same committed requirements both to guide the remaining checks and to reject conclusions the recorded checks do not authorize. This targets a narrower property than verifying all intermediate reasoning or the factual correctness of the observations.

More directly, ResearchLoop \citep{researchloop2026} uses durable task contracts, evidence objects, and claim-admission rules for AI-assisted research. Proof-or-Stop \citep{prooforstop2026} requires fresh, mechanically admissible evidence before advancing coding lifecycle states and includes repair or stopping. EG-VAR \citep{egvar2026} uses tool-attested proofs to certify empirical claims. These systems already demonstrate evidence gating; PRIME \emph{does not claim to originate evidence gating or precommitment}. Its more specific control problem asks, after partial or failed observations, which \emph{graded conclusion} is currently authorized and which \emph{still-live sufficient route} can potentially establish a stronger one within the remaining check budget. The committed route representation is used for both answers.

HESP \citep{liu2026hesp} is particularly relevant because it combines cost-aware probing with evidence-backed stopping. Its Bayesian hypothesis ledger and information-gain probe selection address diagnostic triage. PRIME instead represents conditions under which particular conclusions are logically sufficient, including disjunctive routes and weaker authorized alternatives. Likewise, GLEAN's calibrated uncertainty is different from PRIME's route-based authorization. These approaches are complementary. Section~\ref{sec:exp:extsys} runs adapted versions of both, with AgentSpec, ResearchLoop, and Proof-or-Stop, on the same tasks, requirements, and budgets. Appendix~\ref{app:extsys} reports the results.

The formal ingredients are classical. Minimal sufficient routes are related to prime implicants \citep{quine1952}, consistency-based diagnosis and truth-maintenance systems \citep{reiter1987,dekleer1986atms}, and abductive classifier explanations \citep{shih2018symbolic,ignatiev2019abduction}. Budgeted checks are related to sequential testing \citep{unluyurt2004sequential,deshpande2016approximation} and active feature-value acquisition \citep{saartsechansky2009active}. PRIME claims none of these ideas in isolation. Its contribution is a \emph{shared precommitted program for route-aware evidence acquisition and graded answer authorization}, together with a layered evaluation showing where enforcement succeeds and where incorrect programs or fallible resolvers still cause errors. In operational reviews, this distinction matters when blocking an unsupported claim is necessary but insufficient: another acceptable evidence route may still be worth checking before reporting a weaker result or abstaining.

\end{document}